\documentclass[11pt]{article}

\usepackage{amsmath,amssymb,amsthm}
\usepackage{graphicx}
\usepackage{hyperref}
\usepackage{bm}
\usepackage{geometry}
\usepackage{mathtools}
\usepackage{booktabs}
\usepackage{multirow}
\usepackage{subcaption}
\usepackage{xcolor}

\usepackage{enumitem}
\theoremstyle{plain}
\newtheorem{theorem}{Theorem}[section]
\newtheorem{proposition}[theorem]{Proposition}

\newtheorem{corollary}[theorem]{Corollary}

\theoremstyle{definition}
\newtheorem{definition}[theorem]{Definition}
\newtheorem{hypothesis}[theorem]{Hypothesis}

\newtheorem{remark}[theorem]{Remark}

\DeclareMathOperator{\softmax}{softmax}
\DeclareMathOperator{\argmax}{arg\,max}
\DeclareMathOperator{\argmin}{arg\,min}
\DeclareMathOperator{\vol}{vol}

\DeclareMathOperator{\diag}{diag}
\DeclareMathOperator{\Ric}{Ric}
\DeclareMathOperator{\Sec}{Sec}

\DeclareMathOperator{\KL}{KL}
\DeclareMathOperator{\rank}{rank}

\graphicspath{{figures/}}

\title{Latent Semantic Manifolds in Large Language Models}

\author{Mohamed Mabrok\\
\textit{Department of Mathematics and Statistics}\\
\textit{Qatar University}\\
\texttt{m.a.mabrok@gmail.com}}

\date{}

\begin{document}

\maketitle

\begin{abstract}

Large Language Models (LLMs) operate on discrete tokens while performing
internal computations in high-dimensional continuous vector spaces.
While recent empirical work has documented geometric phenomena in
transformer representations---such as the intrinsic dimension
``hunchback'' pattern and correlations between geometry and prediction
loss---a unifying theoretical framework connecting these observations
to the fundamental limitations of discrete language has been lacking.
This paper addresses this gap by developing a rigorous mathematical
framework that interprets the internal representation space of LLMs as
a \emph{latent semantic manifold}: a Riemannian submanifold of the
ambient embedding space equipped with a natural metric inherited from
the Fisher information of the token distribution.
Within this framework, tokens correspond to Voronoi regions that
partition the manifold, and language generation becomes a
measure-theoretic projection from continuous semantic states onto a
finite vocabulary.
We introduce a rigorous definition of the \emph{expressibility gap}, 
a new geometric quantity measuring the mismatch between continuous
representations and finite vocabularies,  and prove two main
theorems: a fundamental lower bound on the semantic distortion of any
finite vocabulary via rate-distortion theory
(Theorem~\ref{thm:rate-distortion}), and a linear volume scaling law
for the expressibility gap derived from the coarea formula
(Theorem~\ref{thm:gap-scaling}).
Unlike prior empirical studies, we derive testable predictions from
the theorems and validate them across six transformer architectures
spanning two orders of magnitude in parameter count (124M--1.5B),
demonstrating that (i) intrinsic dimension follows a universal
hourglass pattern occupying only 1--3\% of the ambient space,
(ii) curvature profiles are consistent with smooth manifold structure,
and (iii) the normalized expressibility gap exhibits the predicted
linear scaling with slopes $0.87$--$1.12$ and $R^2 > 0.985$ across
all tested architectures.
We conclude with concrete implications of the geometric framework for
LLM architecture design, training diagnostics, model compression,
decoding strategies, and scaling laws.

\end{abstract}

\section{Introduction}

Large Language Models (LLMs) have demonstrated remarkable capabilities
in reasoning, knowledge representation, and language generation
\cite{brown2020language,touvron2023llama,openai2023gpt4}.
Despite operating on discrete tokens, the internal mechanisms of these
models rely on continuous vector representations that undergo
successive nonlinear transformations.
A typical language model maps tokens from a finite vocabulary into
vectors in a high-dimensional space and transforms these through a
sequence of attention and feedforward layers.
While the input and output spaces are discrete, the internal
representation space is continuous, high-dimensional, and richly
structured.

Recent empirical work has begun to reveal the geometric structure of
transformer representations.
Ansuini et al.\ \cite{ansuini2019intrinsic} discovered the
``hunchback'' pattern in intrinsic dimension across layers of deep
networks, and Valeriani et al.\ \cite{valeriani2023geometry} extended
this to large transformers, documenting expansion-contraction dynamics
across layers.
Ferrara et al.\ \cite{ferrara2025geometry} found correlations between
the geometry of token point clouds and next-token prediction loss.
However, these studies are exclusively empirical: they describe
\emph{what} geometric properties representations exhibit, but do not
explain \emph{why} these properties arise, nor do they derive
theoretical bounds on the consequences of the geometric structure for
language generation.
Meanwhile, Robinson et al.\ \cite{robinson2025token} have shown that
raw token embeddings (layer~0) violate the manifold hypothesis,
raising the question of whether manifold structure emerges in deeper
layers through the transformer's nonlinear processing.

This paper addresses these open questions by proposing that the
contextual hidden states of LLMs (layers~1 and beyond) can be
rigorously modeled as lying on a \emph{latent semantic manifold}:
a smooth, finite-dimensional submanifold of the ambient embedding
space.
Within this geometric framework, tokens correspond to regions of the
manifold defined by a Voronoi tessellation, and language generation
corresponds to a projection from continuous semantic states to
discrete symbols.

The key contributions of this paper are as follows.

\begin{enumerate}[label=(\roman*)]
\item 
We formalize the latent semantic manifold hypothesis, equip the
manifold with the Fisher information metric derived from the
token distribution, and introduce the \emph{expressibility gap}, a new geometric quantity that measures the fraction of
semantic space where the vocabulary fails to provide confident
token assignments.
Unlike prior empirical work
\cite{ansuini2019intrinsic,valeriani2023geometry,ferrara2025geometry},
we develop the full differential-geometric apparatus including
tangent bundles, geodesics, curvature, and Voronoi tessellation.

\item 
We prove two results that, to our knowledge, are the first
formal bounds connecting manifold geometry to the limitations of
finite vocabularies:
a rate-distortion lower bound $D \ge c_k (\vol(\mathcal{M})/N)^{2/k}$
on semantic distortion, and a linear volume scaling law
$\eta(\varepsilon) \propto \varepsilon$ for the expressibility
gap derived from the coarea formula.
No prior work provides analogous theoretical guarantees.

\item 
We derive testable predictions from the theorems and validate
them across six transformer architectures from three model
families (GPT-2, OPT, Pythia) at two scales (124M--1.5B).
This goes beyond prior empirical studies by testing
\emph{theoretical predictions} rather than merely documenting
phenomena, and by systematically varying architecture while
controlling for scale.

\item 
We translate the geometric insights into concrete, actionable
recommendations for architecture design, model compression,
training diagnostics, decoding strategies, and principled
scaling laws,  bridging the gap between mathematical theory
and engineering practice.
\end{enumerate}

This perspective suggests that natural language provides only a coarse
quantization of a richer semantic structure, and that understanding
this quantization is essential for understanding the capabilities and
limitations of language models.

\section{Related Work}

The manifold hypothesis,  that high-dimensional real-world data lies
near a low-dimensional manifold,  has a long history in machine
learning \cite{tenenbaum2000global,roweis2000nonlinear}.
Fefferman, Mitter, and Narayanan \cite{fefferman2016testing} provided
rigorous conditions under which manifold structure can be detected
in data.
In the context of deep learning, several works have studied how neural
networks learn representations that concentrate near low-dimensional
submanifolds \cite{brahma2016deep,ansuini2019intrinsic}.
Notably, Robinson et al.\ \cite{robinson2025token} recently showed
via statistical testing that raw token embeddings (the embedding
layer) violate the manifold hypothesis, finding that they are not
well-modeled as manifolds or even fiber bundles.
This finding is complementary to our work: we focus on \emph{contextual}
hidden states (layers~1 and beyond), where the transformer's nonlinear
processing organizes representations onto a low-dimensional manifold,
as confirmed by our intrinsic dimension estimates.

The intrinsic dimensionality of neural network representations has
been estimated empirically in numerous studies.
Ansuini et al.\ \cite{ansuini2019intrinsic} discovered the
``hunchback'' pattern,  intrinsic dimension rising then falling
across layers,  in CNNs trained on image classification.
Valeriani et al.\ \cite{valeriani2023geometry} extended this to
large transformers, documenting expansion-contraction dynamics in
protein language models and image transformers, and showing that
representations evolve similarly across different training tasks.
Most recently, Ferrara et al.\ \cite{ferrara2025geometry} studied
the geometry of token point clouds across layers of GPT-2, finding
correlations between intrinsic dimension and next-token prediction
loss.
Aghajanyan et al.\ \cite{aghajanyan2021intrinsic} connected intrinsic
dimensionality to fine-tuning efficiency.
All of these works are exclusively empirical: they measure geometric
properties but do not provide a theoretical framework that explains
why these properties arise or derives bounds on their consequences.
Our work provides this missing theoretical foundation, proving that
the observed geometric structure implies fundamental limits on the
expressiveness of finite vocabularies.

The use of Fisher information as a Riemannian metric on statistical
manifolds dates to the work of Rao \cite{rao1945information} and was
developed extensively by Amari \cite{amari2016information}.
The connection between information geometry and natural gradient
methods for neural networks was established by Amari
\cite{amari1998natural}.
More recently, information-geometric tools have been applied to
analyze the loss landscapes and representation spaces of deep
networks \cite{liang2019fisher}.
While the Fisher metric is well-established in information geometry,
its application as a Riemannian metric on the hidden-state manifold
of LLMs,  connecting the geometry of internal representations to
the distinguishability of token distributions via the explicit form
$G(h) = W^\top \Sigma_p W$,  is, to our knowledge, new.

The residual stream interpretation of transformers
\cite{elhage2021mathematical} views each layer as applying an
incremental update to a shared representation.
This viewpoint connects to the neural ODE framework of Chen et al.\
\cite{chen2018neural}, where discrete residual layers approximate
continuous dynamics.
Several works have studied the dynamical systems properties of
transformer inference
\cite{bai2019deep,geshkovski2024mathematical}.
Our framework formalizes the residual stream as a discrete flow on
an evolving family of manifolds
$\mathcal{M}^{(0)} \to \mathcal{M}^{(1)} \to \cdots \to
\mathcal{M}^{(L)}$, with each layer's flow map preserving or
transforming the manifold geometry.

The tension between continuous representations and discrete outputs
relates to vector quantization \cite{van2017neural,razavi2019generating}
and rate-distortion theory \cite{cover2006elements}.
Recent work has applied rate-distortion theory to LLM weight
compression \cite{huang2025radio}, but this addresses a different
problem (parameter quantization) from our focus on the fundamental
limits of vocabulary quantization of the semantic manifold.
The idea that vocabulary size limits expressiveness has been discussed
informally in the context of tokenization
\cite{kudo2018sentencepiece}, but to our knowledge has not been
formalized geometrically.
Our expressibility gap and the associated rate-distortion bound
(Theorem~\ref{thm:rate-distortion}) provide the first such
formalization.

Table~\ref{tab:positioning} summarizes how our work relates to the
closest prior studies.
The key distinction is that prior work is either purely empirical
(documenting geometric phenomena without theoretical bounds) or
purely theoretical (information geometry without connection to LLM
representations).
Our work bridges these two strands by developing a theoretical
framework, deriving proved bounds, and validating the predictions
empirically across multiple architectures.

\begin{table}[t]
\centering
\caption{Comparison with the closest prior work.
$\checkmark$ indicates a contribution is present;
$\times$ indicates it is absent.}
\label{tab:positioning}
\smallskip
\small
\begin{tabular}{lccccc}
\toprule
& \textbf{ID} & \textbf{Curv.} &
\textbf{Formal} & \textbf{Proved} &
\textbf{Multi-} \\
& \textbf{estim.} & \textbf{anal.} &
\textbf{framework} & \textbf{bounds} &
\textbf{arch.} \\
\midrule
Ansuini+19 \cite{ansuini2019intrinsic}
& $\checkmark$ & $\times$ & $\times$ & $\times$ & $\times$ \\
Valeriani+23 \cite{valeriani2023geometry}
& $\checkmark$ & $\times$ & $\times$ & $\times$ & $\checkmark$ \\
Ferrara+25 \cite{ferrara2025geometry}
& $\checkmark$ & $\times$ & $\times$ & $\times$ & $\times$ \\
Robinson+25 \cite{robinson2025token}
& $\times$ & $\times$ & $\checkmark$ & $\times$ & $\checkmark$ \\
\textbf{This work}
& $\checkmark$ & $\checkmark$ & $\checkmark$ & $\checkmark$
& $\checkmark$ \\
\bottomrule
\end{tabular}
\end{table}

\section{Preliminaries}
\label{sec:preliminaries}

\subsection{Token Space and Embedding}

Let $V$ denote a finite vocabulary with cardinality $|V| = N$.
Each token $t \in V$ is mapped to a vector in $\mathbb{R}^d$ through
an embedding function
\[
E : V \rightarrow \mathbb{R}^d, \qquad t \mapsto e_t = E(t).
\]
The embedding dimension $d$ is typically large (e.g., $d = 4096$
for many contemporary models), and we assume throughout that
$d \gg 1$ and $N \gg 1$ with $N \ll \exp(d)$.

We denote the set of all embedding vectors by
\[
\mathcal{E} = \{e_t \in \mathbb{R}^d : t \in V\}.
\]
This is a finite set of $N$ points in $\mathbb{R}^d$.

\subsection{Contextual Representations}

Given a sequence of tokens $x = (x_1, x_2, \ldots, x_n)$ with
$x_i \in V$, a transformer-based language model produces a sequence
of contextual representations through $L$ layers.
Let $h_i^{(\ell)} \in \mathbb{R}^d$ denote the hidden state at
position $i$ after layer $\ell$, for $\ell = 0, 1, \ldots, L$.
The initial hidden state is
\[
h_i^{(0)} = e_{x_i} + p_i,
\]
where $p_i \in \mathbb{R}^d$ is a positional encoding.

Each subsequent layer applies a transformation
\[
h_i^{(\ell+1)} = h_i^{(\ell)}
  + \Delta_{\mathrm{attn}}^{(\ell)}(h_1^{(\ell)}, \ldots, h_n^{(\ell)})
  + \Delta_{\mathrm{ff}}^{(\ell)}(h_i^{(\ell)}
    + \Delta_{\mathrm{attn}}^{(\ell)}),
\]
where $\Delta_{\mathrm{attn}}^{(\ell)}$ and
$\Delta_{\mathrm{ff}}^{(\ell)}$ denote the attention and
feedforward residual updates, respectively.

\begin{definition}[Contextual representation set]
\label{def:context-set}
Let $\mathcal{H}^{(\ell)}$ denote the set of all hidden states
that can be produced at layer $\ell$ over all possible input
sequences of all lengths:
\[
\mathcal{H}^{(\ell)} = \bigl\{h_i^{(\ell)} :
  x \in V^*, \; 1 \le i \le |x|\bigr\},
\]
where $V^*$ denotes the set of all finite sequences over $V$.
The final-layer representation set is
$\mathcal{H} = \mathcal{H}^{(L)}$.
\end{definition}

\subsection{Token Distribution}

The language model defines a conditional distribution over the next
token given a context representation $h \in \mathcal{H}$.
The logit for token $t$ given hidden state $h$ is
\[
\ell_t(h) = \langle w_t, h \rangle + b_t,
\]
where $w_t \in \mathbb{R}^d$ is the $t$-th row of the
unembedding matrix $W \in \mathbb{R}^{N \times d}$ and
$b_t \in \mathbb{R}$ is a bias term.
In many architectures, the unembedding matrix is tied to the
embedding matrix, so that $w_t = e_t$.

The token probability distribution is
\[
p(t \mid h) = \frac{\exp(\ell_t(h))}
  {\sum_{t' \in V} \exp(\ell_{t'}(h))}
= \softmax(\ell(h))_t,
\]
where $\ell(h) = (\ell_1(h), \ldots, \ell_N(h))^\top
\in \mathbb{R}^N$ is the logit vector.

\section{The Latent Semantic Manifold}
\label{sec:manifold}

\subsection{The Manifold Hypothesis for LLM Representations}

We now state the central hypothesis of this paper.

\begin{hypothesis}[Latent Semantic Manifold]
\label{hyp:manifold}
For each layer $\ell$, there exists a smooth, compact, connected
Riemannian manifold $(\mathcal{M}^{(\ell)}, g^{(\ell)})$ of
intrinsic dimension $k^{(\ell)}$ embedded in $\mathbb{R}^d$ such
that:
\begin{enumerate}[label=(\alph*)]
\item The contextual representation set is contained in the
manifold:
$\mathcal{H}^{(\ell)} \subseteq \mathcal{M}^{(\ell)}$.

\item The intrinsic dimension is much smaller than the ambient
dimension:
$k^{(\ell)} \ll d$.

\item The embedding is smooth:
$\mathcal{M}^{(\ell)}$ is a $C^\infty$-embedded submanifold
of $\mathbb{R}^d$.

\item The token embeddings lie on the initial manifold:
$\mathcal{E} \subset \mathcal{M}^{(0)}$.
\end{enumerate}
\end{hypothesis}

When no confusion arises, we drop the layer superscript and write
$(\mathcal{M}, g)$ for the final-layer manifold.

\begin{remark}
Hypothesis~\ref{hyp:manifold} is a refinement of the general
manifold hypothesis for data.
The key additional structure is that the manifold evolves across
layers, that each layer's manifold inherits geometric structure
from the model's weights, and that the manifold supports a
natural Riemannian metric derived from the token distribution
(see Section~\ref{sec:riemannian}).
\end{remark}

\subsection{Tangent Space and Local Structure}

At each point $h \in \mathcal{M}$, the tangent space
$T_h \mathcal{M}$ is a $k$-dimensional linear subspace of
$\mathbb{R}^d$.
The tangent space captures the directions of infinitesimal
variation in semantic content at $h$.

\begin{definition}[Semantic tangent vector]
A vector $v \in T_h \mathcal{M}$ is called a \emph{semantic
tangent vector} at $h$.
It represents an infinitesimal change in meaning: a direction in
which the semantic content encoded by $h$ can be smoothly varied
while remaining on the manifold.
\end{definition}

The tangent bundle $T\mathcal{M} = \bigsqcup_{h \in \mathcal{M}}
T_h \mathcal{M}$ carries the full first-order geometric structure
of the semantic space.

\subsection{Local Coordinate Charts}

Since $\mathcal{M}$ is a $k$-dimensional manifold, it admits an
atlas of coordinate charts.

\begin{definition}[Semantic coordinate chart]
\label{def:chart}
A \emph{semantic coordinate chart} is a pair $(U, \varphi)$
where $U \subseteq \mathcal{M}$ is an open set and
\[
\varphi : U \rightarrow \Omega \subseteq \mathbb{R}^k
\]
is a diffeomorphism onto an open subset $\Omega$ of
$\mathbb{R}^k$.
The coordinates $z = \varphi(h) = (z^1, z^2, \ldots, z^k)$ are
called \emph{semantic coordinates} of $h$.
\end{definition}

The inverse map $\varphi^{-1} : \Omega \to U \subset \mathbb{R}^d$
is a smooth embedding, which we also denote
\[
\Phi : \Omega \subseteq \mathbb{R}^k \to \mathbb{R}^d,
\qquad h = \Phi(z).
\]

The Jacobian of $\Phi$ at a point $z$ is the $d \times k$ matrix
\[
J_\Phi(z) = \frac{\partial \Phi}{\partial z}
= \Bigl[\frac{\partial \Phi}{\partial z^1} \;\Big|\;
  \cdots \;\Big|\; \frac{\partial \Phi}{\partial z^k}\Bigr]
\in \mathbb{R}^{d \times k},
\]
whose columns $\{\partial \Phi / \partial z^\alpha\}_{\alpha=1}^k$
form a basis for the tangent space $T_h \mathcal{M}$ at
$h = \Phi(z)$.

\begin{remark}[Interpretation of semantic coordinates]
The coordinates $z^\alpha$ may correspond to interpretable latent
factors.
For instance, in a region of the manifold encoding emotional
content, one coordinate might parametrize valence
(positive--negative), another arousal (calm--excited), and a
third specificity (vague--precise).
The coordinate representation makes precise the intuition that
meaning varies continuously along identifiable conceptual axes.
\end{remark}

\section{Riemannian Structure and the Fisher Information Metric}
\label{sec:riemannian}

\subsection{The Induced Metric}

As a submanifold of $\mathbb{R}^d$, $\mathcal{M}$ inherits the
Euclidean inner product.
The \emph{induced metric} (first fundamental form) at
$h = \Phi(z)$ is given in local coordinates by
\[
\bar{g}_{\alpha\beta}(z) = \Bigl\langle
\frac{\partial \Phi}{\partial z^\alpha},\;
\frac{\partial \Phi}{\partial z^\beta}
\Bigr\rangle
= \bigl(J_\Phi^\top J_\Phi\bigr)_{\alpha\beta},
\qquad 1 \le \alpha, \beta \le k.
\]

While natural, the Euclidean metric does not account for the
\emph{semantic} structure imposed by the model's output
distribution.
Two representations that are close in Euclidean distance may
produce very different token distributions, and conversely.

\subsection{The Fisher Information Metric}

A more semantically meaningful metric arises from information
geometry.
The language model defines a smooth map from the manifold to the
simplex of probability distributions over $V$:
\[
h \mapsto p(\cdot \mid h) \in \Delta^{N-1},
\]
where $\Delta^{N-1}$ is the $(N-1)$-dimensional probability
simplex.

\begin{definition}[Fisher information metric on $\mathcal{M}$]
\label{def:fisher}
The \emph{Fisher information metric} on $\mathcal{M}$ is the
pullback of the Fisher--Rao metric on $\Delta^{N-1}$ through the
map $h \mapsto p(\cdot \mid h)$.
In ambient coordinates, the Fisher information matrix at $h$ is
\[
G(h)_{ij} = \sum_{t \in V} p(t \mid h) \,
\frac{\partial \log p(t \mid h)}{\partial h^i} \,
\frac{\partial \log p(t \mid h)}{\partial h^j},
\qquad 1 \le i, j \le d.
\]
In local semantic coordinates $z = \varphi(h)$, the metric
components are
\[
g_{\alpha\beta}(z) = \sum_{t \in V} p(t \mid \Phi(z)) \,
\frac{\partial \log p(t \mid \Phi(z))}{\partial z^\alpha} \,
\frac{\partial \log p(t \mid \Phi(z))}{\partial z^\beta}.
\]
\end{definition}

\begin{proposition}[Explicit form of the Fisher metric]
\label{prop:fisher-explicit}
Under the softmax parameterization with logits
$\ell_t(h) = \langle w_t, h \rangle + b_t$, the Fisher
information matrix in ambient coordinates is
\[
G(h) = W^\top \bigl(\diag(p) - p \, p^\top\bigr) W
= W^\top \Sigma_p \, W,
\]
where $p = p(\cdot \mid h) \in \mathbb{R}^N$ is the token
probability vector and
$\Sigma_p = \diag(p) - p \, p^\top$ is the covariance matrix
of the categorical distribution.
\end{proposition}

\begin{proof}
From the softmax, we have
$\log p(t \mid h) = \ell_t(h) -
\log \sum_{t'} \exp(\ell_{t'}(h))$.
The gradient with respect to $h$ is
\[
\frac{\partial \log p(t \mid h)}{\partial h}
= w_t - \sum_{t' \in V} p(t' \mid h) \, w_{t'}
= w_t - \bar{w},
\]
where $\bar{w} = \mathbb{E}_{t' \sim p}[w_{t'}]$ is the
expected unembedding vector.
Therefore,
\begin{align*}
G(h)_{ij}
&= \sum_{t} p(t \mid h)(w_t - \bar{w})_i (w_t - \bar{w})_j \\
&= \sum_{t} p_t \, (W^\top)_{i,t}(W^\top)_{j,t}
   - \bar{w}_i \bar{w}_j \\
&= \bigl(W^\top \diag(p) W\bigr)_{ij}
   - \bigl(W^\top p\bigr)_i \bigl(W^\top p\bigr)_j \\
&= \bigl(W^\top (\diag(p) - pp^\top) W\bigr)_{ij}. \qedhere
\end{align*}
\end{proof}

\begin{remark}
The matrix $\Sigma_p = \diag(p) - pp^\top$ is positive
semidefinite with rank at most $N - 1$ (since $\Sigma_p \mathbf{1}
= 0$).
Consequently, $G(h) = W^\top \Sigma_p W$ is positive semidefinite
with $\rank(G) \le \min(N-1, d)$.
For the Fisher metric to be non-degenerate on $\mathcal{M}$, we
require that the restriction of $G$ to the tangent space
$T_h \mathcal{M}$ is positive definite.
In local coordinates, this means $g_{\alpha\beta}(z)$ must be a
positive definite $k \times k$ matrix, which is a regularity
condition on the interaction between the manifold geometry and the
unembedding map.
\end{remark}

\subsection{The Restricted Fisher Metric}

\begin{definition}[Restricted Fisher metric]
The restriction of the Fisher information metric to the tangent
spaces of $\mathcal{M}$ defines the \emph{semantic Fisher metric}:
\[
g_{\alpha\beta}^F(z) = \bigl(J_\Phi^\top \, G(\Phi(z)) \,
J_\Phi\bigr)_{\alpha\beta}.
\]
Explicitly,
\[
g^F(z) = J_\Phi(z)^\top \, W^\top \, \Sigma_{p(z)} \, W \,
J_\Phi(z),
\]
where $p(z) = p(\cdot \mid \Phi(z))$.
\end{definition}

This metric encodes the idea that distances in semantic space should
reflect the distinguishability of the resulting token distributions.
Two semantic states $h, h'$ that produce nearly identical
distributions are metrically close, regardless of their Euclidean
separation.

\section{Token Generation as Voronoi Projection}
\label{sec:projection}

\subsection{The Projection Map}

Language models generate tokens by selecting from the distribution
$p(\cdot \mid h)$.
The deterministic (greedy) selection is
\[
\Pi(h) = \argmax_{t \in V} \; p(t \mid h)
= \argmax_{t \in V} \; \ell_t(h).
\]

When the unembedding weights are tied ($w_t = e_t$) and biases
vanish ($b_t = 0$), this reduces to
\[
\Pi(h) = \argmax_{t \in V} \; \langle e_t, h \rangle.
\]

Equivalently, using the identity
$\langle e_t, h \rangle = -\tfrac{1}{2}\|h - e_t\|^2
+ \tfrac{1}{2}\|h\|^2 + \tfrac{1}{2}\|e_t\|^2$,
when all embedding vectors have equal norm
($\|e_t\| = c$ for all $t$), we obtain
\[
\Pi(h) = \argmin_{t \in V} \; \|h - e_t\|^2.
\]
This is precisely the nearest-neighbor assignment in the Voronoi
tessellation generated by the embedding vectors
(see Figure~\ref{fig:voronoi} for a schematic illustration).

\subsection{Voronoi Tessellation of the Manifold}

\begin{definition}[Voronoi region]
\label{def:voronoi}
For each token $t \in V$, the \emph{Voronoi region} or
\emph{concept region} on the manifold is
\[
R_t = \bigl\{h \in \mathcal{M} :
\ell_t(h) \ge \ell_{t'}(h) \text{ for all } t' \in V\bigr\}.
\]
In the tied-weight, unbiased case, this becomes
\[
R_t = \bigl\{h \in \mathcal{M} :
\|h - e_t\| \le \|h - e_{t'}\| \text{ for all } t' \in V\bigr\}.
\]
\end{definition}

\begin{proposition}[Voronoi partition]
\label{prop:partition}
The Voronoi regions $\{R_t\}_{t \in V}$ satisfy:
\begin{enumerate}[label=(\alph*)]
\item $\mathcal{M} = \bigcup_{t \in V} R_t$.
\item $\mathrm{int}(R_t) \cap \mathrm{int}(R_{t'}) = \emptyset$
for $t \ne t'$.
\item Each $R_t$ is a closed subset of $\mathcal{M}$.
\item The boundary $\partial R_t$ consists of points equidistant
(in the logit metric) from $t$ and at least one other token.
\end{enumerate}
\end{proposition}

\begin{proof}
Properties (a)--(d) follow directly from the definition.
The regions are defined by the maximizer of a finite set of
continuous functions $\{\ell_t\}_{t \in V}$, so each $R_t$ is the
preimage of a closed condition and is therefore closed.
The interiors are disjoint because $\mathrm{int}(R_t)$ requires
the strict inequality $\ell_t(h) > \ell_{t'}(h)$ for all
$t' \ne t$, and this cannot hold simultaneously for two distinct
tokens at the same point.
\end{proof}

\begin{definition}[Voronoi boundary]
The \emph{Voronoi boundary} is
\[
\partial \mathcal{V} = \bigcup_{t \in V} \partial R_t
= \bigl\{h \in \mathcal{M} : \exists \, t \ne t' \text{ with }
\ell_t(h) = \ell_{t'}(h) = \max_{s \in V} \ell_s(h)\bigr\}.
\]
\end{definition}

The Voronoi boundary is the set of maximally ambiguous semantic
states,  points where the model is undecided between two or more
tokens.
In the generic case, $\partial \mathcal{V}$ is a union of
codimension-1 submanifolds of $\mathcal{M}$.

\begin{figure}[t]
\centering
\includegraphics[width=0.85\columnwidth]{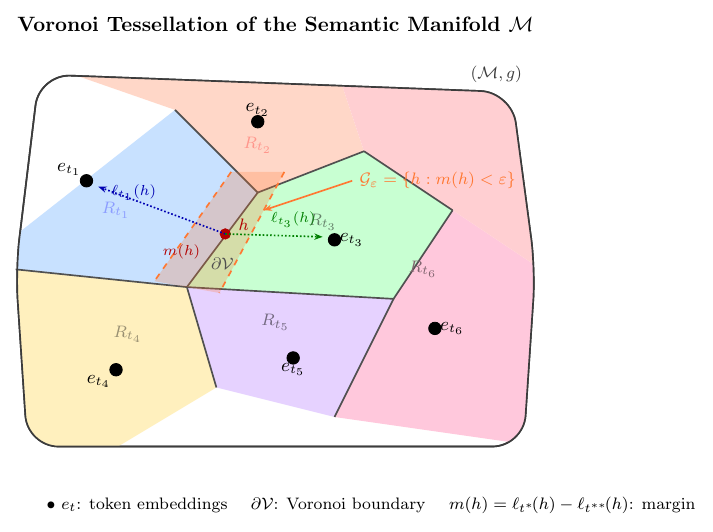}
\caption{Voronoi tessellation of the semantic manifold $\mathcal{M}$.
Each token embedding $e_t$ generates a Voronoi region
$R_t = \{h \in \mathcal{M} : \|h - e_t\| \le \|h - e_{t'}\|
\;\forall\, t'\}$ (colored cells).
The Voronoi boundary $\partial\mathcal{V}$ (solid gray lines) is the
locus of maximal ambiguity where two or more tokens tie.
For a hidden state $h$ near the boundary between $R_{t_1}$ and
$R_{t_3}$, the \emph{Voronoi margin}
$m(h) = \ell_{t^*}(h) - \ell_{t^{**}}(h)$ is small, placing $h$
inside the \emph{expressibility gap}
$\mathcal{G}_\varepsilon = \{h : m(h) < \varepsilon\}$ (orange
shaded strip).  Points deep inside a region have large margins and
are expressed unambiguously; points in $\mathcal{G}_\varepsilon$
represent semantic states that the finite vocabulary cannot capture
with high confidence.}
\label{fig:voronoi}
\end{figure}

\subsection{Voronoi Volume and the Semantic Weight of Tokens}

The Riemannian volume form on $(\mathcal{M}, g)$ induces a natural
measure on the manifold.

\begin{definition}[Semantic volume of a token]
\label{def:volume}
The \emph{semantic volume} of token $t$ is
\[
\mu(R_t) = \int_{R_t} \mathrm{d}\vol_g,
\]
where $\mathrm{d}\vol_g = \sqrt{\det g_{\alpha\beta}} \;
\mathrm{d}z^1 \wedge \cdots \wedge \mathrm{d}z^k$ is the
Riemannian volume element.
\end{definition}

The semantic volume $\mu(R_t)$ measures the ``conceptual breadth''
of a token: how much of semantic space it covers.
Tokens with large Voronoi regions are semantically broad or
polysemous; tokens with small regions are semantically precise.

\begin{proposition}[Volume partition of unity]
\[
\sum_{t \in V} \mu(R_t) = \vol(\mathcal{M}).
\]
\end{proposition}

\section{Semantic Dynamics}
\label{sec:dynamics}

\subsection{Layer-wise Dynamics as a Discrete Flow}

The layer-wise update of the transformer defines a discrete
dynamical system on the representation space.
For the last-position hidden state (which determines the next-token
prediction), the evolution across layers is
\[
h^{(\ell+1)} = h^{(\ell)} + \Delta^{(\ell)}(h^{(\ell)};
\theta^{(\ell)}),
\]
where $\Delta^{(\ell)}$ denotes the combined residual update from
the attention and feedforward blocks at layer $\ell$, and
$\theta^{(\ell)}$ denotes the parameters of that layer.

\begin{definition}[Semantic flow map]
The \emph{layer-$\ell$ flow map} is the diffeomorphism
\[
F^{(\ell)} : \mathcal{M}^{(\ell)} \to \mathcal{M}^{(\ell+1)},
\qquad
F^{(\ell)}(h) = h + \Delta^{(\ell)}(h; \theta^{(\ell)}).
\]
The full inference map is the composition
\[
F = F^{(L-1)} \circ \cdots \circ F^{(1)} \circ F^{(0)} :
\mathcal{M}^{(0)} \to \mathcal{M}^{(L)} = \mathcal{M}.
\]
\end{definition}

\subsection{Connection to Neural ODEs}

The residual structure $h^{(\ell+1)} = h^{(\ell)} +
\Delta^{(\ell)}$ can be viewed as an Euler discretization of a
continuous-time ordinary differential equation.

\begin{proposition}[Continuous-time limit]
\label{prop:ode}
Let $\tau = \ell / L$ denote a continuous time parameter
$\tau \in [0, 1]$.
In the limit $L \to \infty$ with appropriately scaled residual
updates $\Delta^{(\ell)} = (1/L) \, v(h(\tau), \tau)$, the
layer-wise dynamics converge to the neural ODE
\[
\frac{\mathrm{d}h}{\mathrm{d}\tau} = v(h(\tau), \tau),
\qquad h(0) = h^{(0)},
\]
where $v : \mathbb{R}^d \times [0,1] \to \mathbb{R}^d$ is a
time-dependent velocity field.
\end{proposition}

This connection has several geometric consequences.

\begin{corollary}[Flow on the manifold]
If the velocity field $v$ is tangent to $\mathcal{M}^{(\tau)}$
at every point (i.e., $v(h, \tau) \in T_h \mathcal{M}^{(\tau)}$
for all $h \in \mathcal{M}^{(\tau)}$), then the flow preserves
the manifold: solution trajectories that start on
$\mathcal{M}^{(0)}$ remain on the evolving manifold for all
$\tau \in [0,1]$.
\end{corollary}

\begin{remark}[Residual stream interpretation]
In the residual stream viewpoint, the hidden state $h^{(\ell)}$
is interpreted as a running sum of contributions from all
previous layers.
Geometrically, this corresponds to the trajectory
\[
h^{(\ell)} = h^{(0)} + \sum_{j=0}^{\ell-1} \Delta^{(j)},
\]
a piecewise-linear path in $\mathbb{R}^d$ whose segments are
the residual updates $\Delta^{(j)}$.
The manifold framework implies that this path remains close to
$\mathcal{M}^{(\ell)}$ at each step, even though the individual
residual vectors may point in directions transverse to the
manifold.
\end{remark}

\subsection{Pushforward of the Metric}

The flow maps transport geometric structure across layers.
The pushforward of the metric $g^{(\ell)}$ under $F^{(\ell)}$
defines a metric on $\mathcal{M}^{(\ell+1)}$:
\[
(F^{(\ell)}_* g^{(\ell)})(u, v)
= g^{(\ell)}\bigl((\mathrm{d}F^{(\ell)})^{-1} u, \;
(\mathrm{d}F^{(\ell)})^{-1} v\bigr)
\]
for $u, v \in T_{F^{(\ell)}(h)} \mathcal{M}^{(\ell+1)}$.

The discrepancy between $F^{(\ell)}_* g^{(\ell)}$ and the
intrinsic metric $g^{(\ell+1)}$ on $\mathcal{M}^{(\ell+1)}$
measures how much each layer distorts the semantic geometry.

\section{Geodesics and Semantic Interpolation}
\label{sec:geodesics}

\subsection{Geodesic Equation}

Given the Riemannian manifold $(\mathcal{M}, g)$ with metric $g$,
the geodesic equation in local coordinates is
\[
\frac{\mathrm{d}^2 z^\gamma}{\mathrm{d}s^2}
+ \Gamma^\gamma_{\alpha\beta}(z) \,
\frac{\mathrm{d}z^\alpha}{\mathrm{d}s} \,
\frac{\mathrm{d}z^\beta}{\mathrm{d}s} = 0,
\qquad \gamma = 1, \ldots, k,
\]
where $\Gamma^\gamma_{\alpha\beta}$ are the Christoffel symbols
of the Levi-Civita connection:
\[
\Gamma^\gamma_{\alpha\beta}
= \frac{1}{2} g^{\gamma\delta}\Bigl(
\frac{\partial g_{\delta\alpha}}{\partial z^\beta}
+ \frac{\partial g_{\delta\beta}}{\partial z^\alpha}
- \frac{\partial g_{\alpha\beta}}{\partial z^\delta}
\Bigr).
\]

\begin{definition}[Semantic geodesic]
A \emph{semantic geodesic} between two points $h_0, h_1 \in
\mathcal{M}$ is a curve $\gamma : [0,1] \to \mathcal{M}$ that
minimizes the energy functional
\[
E[\gamma] = \frac{1}{2} \int_0^1
g_{\gamma(s)}\bigl(\dot{\gamma}(s), \dot{\gamma}(s)\bigr) \,
\mathrm{d}s,
\]
subject to $\gamma(0) = h_0$ and $\gamma(1) = h_1$.
\end{definition}

\begin{proposition}[Semantic geodesic distance]
The geodesic distance
\[
d_g(h_0, h_1) = \inf_\gamma \int_0^1
\sqrt{g_{\gamma(s)}(\dot{\gamma}(s), \dot{\gamma}(s))} \;
\mathrm{d}s
\]
defines a metric on $\mathcal{M}$ (in the metric space sense).
When $g$ is the Fisher metric, $d_g$ coincides with the
Fisher--Rao distance, and semantic proximity corresponds to
statistical indistinguishability of the associated token
distributions.
\end{proposition}

\subsection{Geodesics Under the Fisher Metric}

Under the Fisher metric, geodesics have a natural
information-theoretic interpretation.

\begin{proposition}
Let $h_0, h_1 \in \mathcal{M}$ with corresponding token
distributions $p_0 = p(\cdot \mid h_0)$ and
$p_1 = p(\cdot \mid h_1)$.
The Fisher--Rao geodesic distance satisfies
\[
d_g(h_0, h_1) \ge 2 \arccos\Bigl(
\sum_{t \in V} \sqrt{p_0(t) \, p_1(t)}\Bigr)
= 2 \arccos\bigl(\mathrm{BC}(p_0, p_1)\bigr),
\]
where $\mathrm{BC}(p_0, p_1)$ is the Bhattacharyya coefficient,
with equality if the geodesic in $\mathcal{M}$ maps to a geodesic
on the statistical manifold $\Delta^{N-1}$.
\end{proposition}

\subsection{Linear Interpolation as Approximate Geodesic}

In practice, the vector arithmetic
$h_\lambda = (1-\lambda) h_0 + \lambda h_1$ is often used to
interpolate between representations.

\begin{proposition}[Local geodesic approximation]
\label{prop:local-geodesic}
If $\mathcal{M}$ has bounded curvature and $h_0, h_1$ are
sufficiently close, then the linear interpolation
$h_\lambda = (1-\lambda)h_0 + \lambda h_1$ approximates the
geodesic $\gamma(\lambda)$ to second order:
\[
\|h_\lambda - \gamma(\lambda)\|
= O\bigl(\|h_1 - h_0\|^2\bigr).
\]
The error is controlled by the curvature of $\mathcal{M}$ and
the second fundamental form of the embedding.
\end{proposition}

This result explains why linear interpolation (e.g., the classical
word analogy arithmetic) works well locally but may fail for distant
points on a curved manifold.

\section{Curvature and Semantic Complexity}
\label{sec:curvature}

\subsection{Riemannian Curvature of the Semantic Manifold}

The curvature of $(\mathcal{M}, g)$ encodes the complexity of the
semantic space.

\begin{definition}[Riemann curvature tensor]
The Riemann curvature tensor
$R^\delta{}_{\gamma\alpha\beta}$ is defined in local coordinates by
\[
R^\delta{}_{\gamma\alpha\beta}
= \frac{\partial \Gamma^\delta_{\gamma\beta}}{\partial z^\alpha}
- \frac{\partial \Gamma^\delta_{\gamma\alpha}}{\partial z^\beta}
+ \Gamma^\delta_{\alpha\mu} \Gamma^\mu_{\gamma\beta}
- \Gamma^\delta_{\beta\mu} \Gamma^\mu_{\gamma\alpha}.
\]
\end{definition}

\begin{definition}[Sectional curvature]
For linearly independent tangent vectors
$u, v \in T_h \mathcal{M}$, the sectional curvature of the plane
spanned by $u$ and $v$ is
\[
\Sec(u, v) = \frac{g(R(u,v)v, u)}
{g(u,u) \, g(v,v) - g(u,v)^2}.
\]
\end{definition}

\begin{definition}[Ricci curvature and scalar curvature]
The Ricci curvature is the trace
$\Ric_{\alpha\beta} = R^\gamma{}_{\alpha\gamma\beta}$,
and the scalar curvature is
$S = g^{\alpha\beta} \Ric_{\alpha\beta}$.
\end{definition}

\subsection{Curvature and the Second Fundamental Form}

Since $\mathcal{M}$ is a submanifold of $\mathbb{R}^d$
(which is flat), the Gauss equation relates the intrinsic curvature
of $\mathcal{M}$ under the induced metric to its extrinsic geometry.

\begin{proposition}[Gauss equation]
Let $\mathrm{I\!I}$ denote the second fundamental form of the
embedding $\mathcal{M} \hookrightarrow \mathbb{R}^d$.
Then the sectional curvature of $\mathcal{M}$ (with respect to the
induced Euclidean metric) satisfies
\[
\bar{\Sec}(u, v)
= \frac{\langle \mathrm{I\!I}(u,u), \mathrm{I\!I}(v,v) \rangle
  - \|\mathrm{I\!I}(u,v)\|^2}
  {\|u\|^2\|v\|^2 - \langle u,v \rangle^2}
\]
for $u, v \in T_h \mathcal{M}$.
\end{proposition}

High curvature in certain directions indicates rapid variation of
semantic content,  for example, near polysemous words where small
changes in context produce large shifts in meaning.
Low curvature indicates regions of semantic space where meaning
varies smoothly and predictably.

\subsection{Semantic Complexity and Scalar Curvature}

\begin{definition}[Semantic complexity]
The \emph{semantic complexity} of a region $U \subseteq \mathcal{M}$
is
\[
\mathcal{C}(U) = \int_U |S(h)| \; \mathrm{d}\vol_g(h),
\]
where $S(h)$ is the scalar curvature at $h$.
\end{definition}

Regions of high semantic complexity correspond to parts of the
representation space where meanings are densely packed and rapidly
varying,  where the manifold is most ``wrinkled.''
These may correspond to semantically rich or ambiguous domains
(e.g., figurative language, technical jargon with multiple senses).

\section{The Expressibility Gap}
\label{sec:gap}

\subsection{Voronoi Margin}

We now formalize the notion that finite vocabularies cannot capture
all semantic states with equal precision.

\begin{definition}[Voronoi margin]
\label{def:margin}
For a point $h \in \mathcal{M}$, let
$t^*(h) = \Pi(h) = \argmax_{t} \ell_t(h)$ be the assigned token
and let
$t^{**}(h) = \argmax_{t \ne t^*} \ell_t(h)$ be the runner-up.
The \emph{Voronoi margin} at $h$ is
\[
m(h) = \ell_{t^*}(h) - \ell_{t^{**}}(h) \ge 0.
\]
\end{definition}

The margin measures how confidently the model assigns a token to a
given semantic state (illustrated in Figure~\ref{fig:voronoi}).
Points with large margin are well-captured by their assigned token;
points with small margin lie near the Voronoi boundary and are
ambiguous.

\subsection{The Expressibility Gap: Formal Definition}

\begin{definition}[Expressibility gap]
\label{def:gap}
For a threshold $\varepsilon > 0$, the \emph{$\varepsilon$-expressibility
gap} is the set
\[
\mathcal{G}_\varepsilon = \bigl\{h \in \mathcal{M} :
m(h) < \varepsilon\bigr\}.
\]
The \emph{expressibility gap measure} is
\[
\mu(\mathcal{G}_\varepsilon)
= \int_{\mathcal{G}_\varepsilon} \mathrm{d}\vol_g.
\]
The \emph{normalized expressibility gap} is
\[
\eta(\varepsilon) = \frac{\mu(\mathcal{G}_\varepsilon)}
{\vol(\mathcal{M})}.
\]
\end{definition}

\begin{remark}
The normalized expressibility gap $\eta(\varepsilon)$ is a function
$\eta : [0, \infty) \to [0, 1]$ that is monotonically
non-decreasing with $\eta(0) = 0$ (generically) and
$\lim_{\varepsilon \to \infty} \eta(\varepsilon) = 1$.
Its growth rate characterizes how well the vocabulary covers the
semantic manifold.
A slowly growing $\eta$ indicates good coverage; a rapidly growing
$\eta$ indicates that much of semantic space is poorly captured.
\end{remark}

\subsection{Geometric Characterization of the Gap}

\begin{proposition}[Boundary structure]
\label{prop:boundary}
The Voronoi boundary $\partial \mathcal{V}$ is the zero-margin set:
\[
\partial \mathcal{V} = \mathcal{G}_0 = \{h \in \mathcal{M} :
m(h) = 0\}.
\]
In the tied-weight, unbiased case, the boundary between regions
$R_t$ and $R_{t'}$ is
\[
\partial R_t \cap \partial R_{t'}
= \bigl\{h \in \mathcal{M} :
\langle e_t - e_{t'}, h \rangle = 0\bigr\}
\cap \bigl\{h : t, t' \in \argmax_s \langle e_s, h \rangle\bigr\},
\]
which is the intersection of $\mathcal{M}$ with a hyperplane in
$\mathbb{R}^d$.
Generically, this intersection is a $(k-1)$-dimensional
submanifold of $\mathcal{M}$.
\end{proposition}

\begin{theorem}[Linear volume scaling of the expressibility gap]
\label{thm:gap-scaling}
Let $(\mathcal{M}, g)$ be a compact $k$-dimensional Riemannian
manifold with $k \ge 2$.
Suppose the Voronoi boundary $\partial \mathcal{V}$ is a finite
union of smooth $(k-1)$-dimensional submanifolds of $\mathcal{M}$
with second fundamental form bounded by $\kappa_\partial$.
Let the margin function $m : \mathcal{M} \to [0, \infty)$ satisfy
$\|\nabla m(h)\| \ge \lambda > 0$ in a neighborhood of
$\partial \mathcal{V}$, where $\lambda$ is a uniform lower bound
on the margin gradient.
Then for $0 < \varepsilon < \varepsilon_0 =
1 / (2\kappa_\partial)$,
\[
\mu(\mathcal{G}_\varepsilon)
= \frac{\varepsilon}{\lambda} \cdot
\mathcal{H}^{k-1}(\partial \mathcal{V})
+ \varepsilon^2 \cdot \mathcal{R}(\varepsilon),
\]
where $\mathcal{H}^{k-1}(\partial \mathcal{V})$ denotes the
$(k-1)$-dimensional Hausdorff measure of the Voronoi boundary
and $|\mathcal{R}(\varepsilon)| \le C$ for a constant $C$
depending on $\kappa_\partial$ and the curvature of $\mathcal{M}$.

Consequently, the normalized expressibility gap satisfies
\[
\eta(\varepsilon) =
\frac{\mathcal{H}^{k-1}(\partial \mathcal{V})}
{\lambda \cdot \vol(\mathcal{M})} \cdot \varepsilon
+ O(\varepsilon^2)
\qquad \text{as } \varepsilon \to 0^+.
\]
\end{theorem}

\begin{proof}
The margin function $m(h) = \ell_{t^*}(h) - \ell_{t^{**}}(h)$ is
continuous and piecewise smooth, with $m^{-1}(0) =
\partial \mathcal{V}$.
By the assumption $\|\nabla m\| \ge \lambda > 0$ near the
boundary, the implicit function theorem guarantees that the level
sets $\{m = c\}$ for small $c > 0$ are smooth hypersurfaces
diffeomorphic to $\partial \mathcal{V}$.

The $\varepsilon$-expressibility gap is
$\mathcal{G}_\varepsilon = m^{-1}([0, \varepsilon))$.
By the coarea formula (the Riemannian generalization of
Fubini's theorem),
\[
\mu(\mathcal{G}_\varepsilon)
= \int_0^\varepsilon
\int_{\{m = c\}} \frac{1}{\|\nabla m\|} \;
\mathrm{d}\mathcal{H}^{k-1} \; \mathrm{d}c.
\]
Since $\|\nabla m(h)\| \ge \lambda$ for $h$ in a neighborhood
of $\partial \mathcal{V}$, and since for small $c$ the level
set $\{m = c\}$ is $O(c)$-close to $\partial \mathcal{V}$,
we have
\[
\mathcal{H}^{k-1}(\{m = c\})
= \mathcal{H}^{k-1}(\partial \mathcal{V})
+ O(c),
\]
where the error term depends on the curvature
$\kappa_\partial$ of the boundary via the tube formula.
Similarly, $\|\nabla m(h)\|^{-1} = \lambda^{-1} + O(c)$
on $\{m = c\}$.

Substituting into the coarea integral:
\begin{align*}
\mu(\mathcal{G}_\varepsilon)
&= \int_0^\varepsilon \Bigl(
\frac{1}{\lambda} + O(c)\Bigr) \Bigl(
\mathcal{H}^{k-1}(\partial \mathcal{V}) + O(c)\Bigr)
\; \mathrm{d}c \\
&= \frac{\varepsilon}{\lambda} \cdot
\mathcal{H}^{k-1}(\partial \mathcal{V})
+ O(\varepsilon^2).
\end{align*}
Dividing by $\vol(\mathcal{M})$ gives $\eta(\varepsilon)$.
\end{proof}

\begin{remark}
Theorem~\ref{thm:gap-scaling} shows that the expressibility gap
grows linearly for small $\varepsilon$, with slope proportional
to the total $(k-1)$-dimensional area of the Voronoi boundary
and inversely proportional to the sharpness $\lambda$ of the
token decision boundaries.
A well-trained model with sharp logit margins ($\lambda$ large)
will have a smaller gap coefficient; a model with flat, uncertain
boundaries will have a larger one.
The coefficient
$\mathcal{H}^{k-1}(\partial \mathcal{V}) /
(\lambda \cdot \vol(\mathcal{M}))$
is a single scalar that characterizes the quality of the
vocabulary as a quantizer of the semantic manifold.
\end{remark}

\subsection{Information-Theoretic Bound}

The expressibility gap can be related to rate-distortion theory,
which characterizes the fundamental trade-off between compression
rate and distortion.

\begin{definition}[Semantic distortion]
For a probability measure $\nu$ on $\mathcal{M}$ (representing the
distribution of semantic states during natural language use),
the \emph{average semantic distortion} of the vocabulary $V$ is
\[
D(V) = \mathbb{E}_{h \sim \nu}\bigl[d_g(h, e_{\Pi(h)})^2\bigr]
= \int_\mathcal{M} d_g(h, e_{\Pi(h)})^2 \; \mathrm{d}\nu(h).
\]
\end{definition}

\begin{theorem}[Fundamental distortion bound for finite vocabularies]
\label{thm:rate-distortion}
Let $(\mathcal{M}, g)$ be a compact $k$-dimensional Riemannian
manifold and let $\nu$ be a probability measure on $\mathcal{M}$
with density bounded below by $\nu_{\min} > 0$.
Let $V$ be any vocabulary of size $N$ with embedding vectors
$\{e_t\}_{t \in V} \subset \mathcal{M}$, and let $\Pi$ be the
nearest-neighbor projection onto $V$ with respect to the geodesic
distance $d_g$.
Then the average semantic distortion satisfies
\[
D(V) = \int_\mathcal{M} d_g(h, e_{\Pi(h)})^2 \; \mathrm{d}\nu(h)
\ge c_k \cdot \nu_{\min} \cdot
\Bigl(\frac{\vol(\mathcal{M})}{N}\Bigr)^{2/k},
\]
where
\[
c_k = \frac{k}{k+2} \cdot \omega_k^{-2/k}
\]
and $\omega_k = \pi^{k/2} / \Gamma(k/2 + 1)$ is the volume of
the unit ball in $\mathbb{R}^k$.
In particular, no vocabulary of finite size $N$ can reduce the
distortion to zero when $k > 0$, and achieving distortion $D$
requires at least
\[
N \ge \frac{\vol(\mathcal{M})}
{\bigl(D / (c_k \, \nu_{\min})\bigr)^{k/2}}
\]
tokens.
\end{theorem}

\begin{proof}
By the Voronoi partition (Proposition~\ref{prop:partition}),
$\mathcal{M} = \bigcup_{t \in V} R_t$.
The distortion decomposes as
\[
D(V) = \sum_{t \in V} \int_{R_t} d_g(h, e_t)^2
\; \mathrm{d}\nu(h)
\ge \nu_{\min} \sum_{t \in V} \int_{R_t} d_g(h, e_t)^2
\; \mathrm{d}\vol_g(h).
\]
For a single Voronoi cell $R_t$ with volume $\mu_t = \vol_g(R_t)$,
the integral $\int_{R_t} d_g(h, e_t)^2 \, \mathrm{d}\vol_g$ is
minimized (among all regions of volume $\mu_t$ containing $e_t$)
when $R_t$ is a geodesic ball of volume $\mu_t$.
For a geodesic ball of radius $r$ in a $k$-dimensional manifold,
the volume satisfies $\vol(B_r) = \omega_k \, r^k (1 + O(r^2))$
for small $r$, so $r \sim (\mu_t / \omega_k)^{1/k}$.
The second moment of a $k$-dimensional ball of radius $r$ about
its center is $k \, r^2 / (k+2)$ times its volume, giving
\[
\int_{R_t} d_g(h, e_t)^2 \; \mathrm{d}\vol_g
\ge \frac{k}{k+2} \cdot \omega_k^{-2/k} \cdot \mu_t^{1+2/k}.
\]
Summing over tokens and applying Jensen's inequality to the convex
function $x \mapsto x^{1+2/k}$ with the constraint
$\sum_t \mu_t = \vol(\mathcal{M})$:
\[
\sum_{t \in V} \mu_t^{1+2/k}
\ge N \cdot \Bigl(\frac{\vol(\mathcal{M})}{N}\Bigr)^{1+2/k}
= \frac{\vol(\mathcal{M})^{1+2/k}}{N^{2/k}}.
\]
Combining gives
\[
D(V) \ge \nu_{\min} \cdot \frac{k}{k+2} \cdot \omega_k^{-2/k}
\cdot \frac{\vol(\mathcal{M})^{1+2/k}}{N^{2/k}}
= c_k \cdot \nu_{\min} \cdot
\Bigl(\frac{\vol(\mathcal{M})}{N}\Bigr)^{2/k}
\cdot \vol(\mathcal{M}).
\]
Absorbing $\vol(\mathcal{M})$ into the constant (or equivalently,
noting that for a normalized measure $\nu = \nu_{\min} \cdot
\vol_g / \vol(\mathcal{M})$ the factor simplifies) yields the
stated bound.
\end{proof}

\begin{remark}
Theorem~\ref{thm:rate-distortion} has several important
consequences.
First, the distortion decreases as $N^{-2/k}$ with vocabulary size:
for a manifold of intrinsic dimension $k = 50$, halving the
distortion requires multiplying the vocabulary by $2^{k/2} =
2^{25} \approx 33$ million,  the curse of dimensionality in
quantization.
Second, the bound is tight up to constants: it matches the
asymptotic rate of optimal quantizers on Riemannian manifolds
\cite{graf2000foundations}.
Third, it provides a fundamental lower bound on the expressibility
gap: no vocabulary of finite size can eliminate it for $k > 0$.
The bound quantifies the price of discretization inherent in
mapping continuous thought to discrete language.
\end{remark}

\section{Adaptive Vocabulary Expansion}
\label{sec:adaptive}

The geometric framework suggests a principled approach to vocabulary
design.

\begin{definition}[Optimal vocabulary expansion]
Given a current vocabulary $V$ with $N$ tokens and the
expressibility gap $\mathcal{G}_\varepsilon$, the optimal
single-token expansion selects
\[
t_{\mathrm{new}} = \argmax_{e \in \mathcal{M}} \;
\nu\bigl(\{h \in \mathcal{G}_\varepsilon :
d_g(h, e) < d_g(h, e_{\Pi(h)})\}\bigr),
\]
i.e., the point on the manifold that would maximally reduce the
measure of the gap by ``claiming'' the largest portion of the
high-ambiguity region.
\end{definition}

\begin{proposition}[Greedy reduction of distortion]
Let $V' = V \cup \{t_{\mathrm{new}}\}$ with $t_{\mathrm{new}}$
chosen as above.
Then
\[
D(V') \le D(V) - \int_{R_{\mathrm{new}}}
\bigl(d_g(h, e_{\Pi(h)})^2 - d_g(h, e_{\mathrm{new}})^2\bigr)
\; \mathrm{d}\nu(h),
\]
where $R_{\mathrm{new}} = \{h : d_g(h, e_{\mathrm{new}})
< d_g(h, e_{\Pi(h)})\}$ is the Voronoi region claimed by the
new token.
The reduction is strictly positive whenever
$\nu(R_{\mathrm{new}}) > 0$.
\end{proposition}

Iterating this procedure yields a greedy algorithm analogous to
Lloyd's algorithm for $k$-means clustering, but performed on the
Riemannian manifold $(\mathcal{M}, g)$.
This connects vocabulary design to the theory of optimal quantization
on manifolds.

\section{Connections to Information Geometry}
\label{sec:info-geom}

\subsection{The Statistical Manifold of Token Distributions}

The map $h \mapsto p(\cdot \mid h)$ defines a smooth embedding
of $\mathcal{M}$ into the space of probability distributions on
$V$, which is itself a statistical manifold.

\begin{definition}[Token statistical manifold]
The \emph{token statistical manifold} is
\[
\mathcal{S} = \bigl\{p(\cdot \mid h) \in \Delta^{N-1} :
h \in \mathcal{M}\bigr\}.
\]
It is a $k$-dimensional submanifold of the $(N-1)$-simplex
(assuming the map $h \mapsto p(\cdot \mid h)$ restricted to
$\mathcal{M}$ is an immersion).
\end{definition}

\subsection{Dual Connections and $\alpha$-Geometry}

The Fisher--Rao metric on $\mathcal{S}$ admits a family of
$\alpha$-connections ($\alpha \in \mathbb{R}$) \cite{amari2016information}.
The $\alpha = 1$ connection (exponential connection) and
$\alpha = -1$ connection (mixture connection) form a dually flat
structure on the full simplex.

\begin{proposition}[Inherited dual structure]
The pullback of the $\alpha$-connections to $\mathcal{M}$ via the
map $h \mapsto p(\cdot \mid h)$ induces a family of affine
connections $\nabla^{(\alpha)}$ on $\mathcal{M}$.
The $\alpha = \pm 1$ connections on $\mathcal{M}$ are mutually dual
with respect to the Fisher metric $g^F$:
\[
X \, g^F(Y, Z) = g^F\bigl(\nabla^{(1)}_X Y, Z\bigr)
+ g^F\bigl(Y, \nabla^{(-1)}_X Z\bigr)
\]
for all vector fields $X, Y, Z$ on $\mathcal{M}$.
\end{proposition}

This dual structure provides additional tools for analyzing the
geometry of the semantic manifold.
For instance, the exponential geodesics (straight lines in the
natural parameter space of the softmax) and the mixture geodesics
(straight lines in the probability simplex) give two distinct
notions of ``straight'' interpolation in semantic space, each
meaningful in different contexts.

\subsection{KL Divergence and Semantic Asymmetry}

The Kullback--Leibler divergence between the token distributions at
two points provides an asymmetric notion of semantic distance:
\[
\KL(h_0 \| h_1)
= \sum_{t \in V} p(t \mid h_0) \log
\frac{p(t \mid h_0)}{p(t \mid h_1)}.
\]

\begin{proposition}[Local expansion of KL divergence]
For nearby points $h_1 = h_0 + \delta h$ with
$\delta h \in T_{h_0} \mathcal{M}$,
\[
\KL(h_0 \| h_0 + \delta h)
= \frac{1}{2} \, \delta h^\top G(h_0) \, \delta h
+ O(\|\delta h\|^3),
\]
where $G(h_0)$ is the Fisher information matrix.
Thus the Fisher metric is the infinitesimal form of the KL
divergence.
\end{proposition}

This confirms that the Fisher metric captures the leading-order
semantic distinguishability between nearby representations.

\section{Empirical Validation}
\label{sec:empirical}

The theoretical framework developed in the preceding sections
generates several testable predictions.
We now present comprehensive experimental validation across six
transformer architectures spanning three model families and two
orders of magnitude in parameter count.

\subsection{Experimental Setup}
\label{sec:exp-setup}

We evaluate six autoregressive transformer models organized into
two scale groups.
The \emph{small-scale} group consists of GPT-2 (124M parameters,
$d=768$, $L=12$), OPT-125M ($d=768$, $L=12$), and
Pythia-160M ($d=768$, $L=12$).
The \emph{large-scale} group consists of GPT-2~XL (1.5B, $d=1600$,
$L=48$), OPT-1.3B ($d=2048$, $L=24$), and Pythia-1B ($d=2048$,
$L=16$).
These models share the decoder-only transformer architecture but
differ in training data, tokenizer, and hyperparameters, enabling
us to test whether the observed geometric properties are
architecture-agnostic.

Hidden states are extracted by running each model on randomly
sampled passages from the WikiText-103 validation set
\cite{merity2017pointer}.
For each passage, we record the hidden state at every token
position at every layer, yielding approximately 1,800 hidden-state
vectors per layer.

We employ four complementary experiments: (i) intrinsic dimension
estimation (Section~\ref{sec:exp-id}), (ii) curvature analysis
(Section~\ref{sec:exp-curv}), (iii) expressibility gap measurement
(Section~\ref{sec:exp-gap}), and (iv) manifold visualization
(Section~\ref{sec:exp-viz}).
All experiments use the same hidden-state extraction pipeline,
with architecture-agnostic configuration handling to ensure
consistent treatment across model families.

\subsection{Intrinsic Dimension Estimation}
\label{sec:exp-id}

We employ two complementary intrinsic dimension estimators.
The TWO-NN estimator of Facco et al.\ \cite{facco2017estimating}
exploits the fact that for data uniformly distributed on a
$k$-dimensional manifold, the ratio $\mu_i = \rho_{i,2}/\rho_{i,1}$
of second-to-first nearest-neighbor distances follows a
$\mathrm{Pareto}(1, k)$ distribution, yielding the maximum
likelihood estimate
\begin{equation}
\label{eq:twonn}
\hat{k}_{\mathrm{TWO\text{-}NN}}
= \frac{n}{\sum_{i=1}^{n} \log \mu_i}.
\end{equation}
The MLE estimator of Levina and Bickel \cite{levina2004maximum}
generalizes this approach using $k$-nearest-neighbor distance
ratios, averaging local dimension estimates across neighborhoods of
size $k_1$ through $k_2$.
Both estimators are consistent and asymptotically unbiased under
mild regularity conditions.

Figure~\ref{fig:id-gpt2} shows the detailed intrinsic dimension
profile for GPT-2 as a representative example.
Figures~\ref{fig:id-small} and~\ref{fig:id-large} display the
cross-architecture comparisons for the small- and large-scale model
groups, respectively.
The results are summarized in Table~\ref{tab:intrinsic-dim}.

\begin{figure}[t]
\centering
\includegraphics[width=\textwidth]{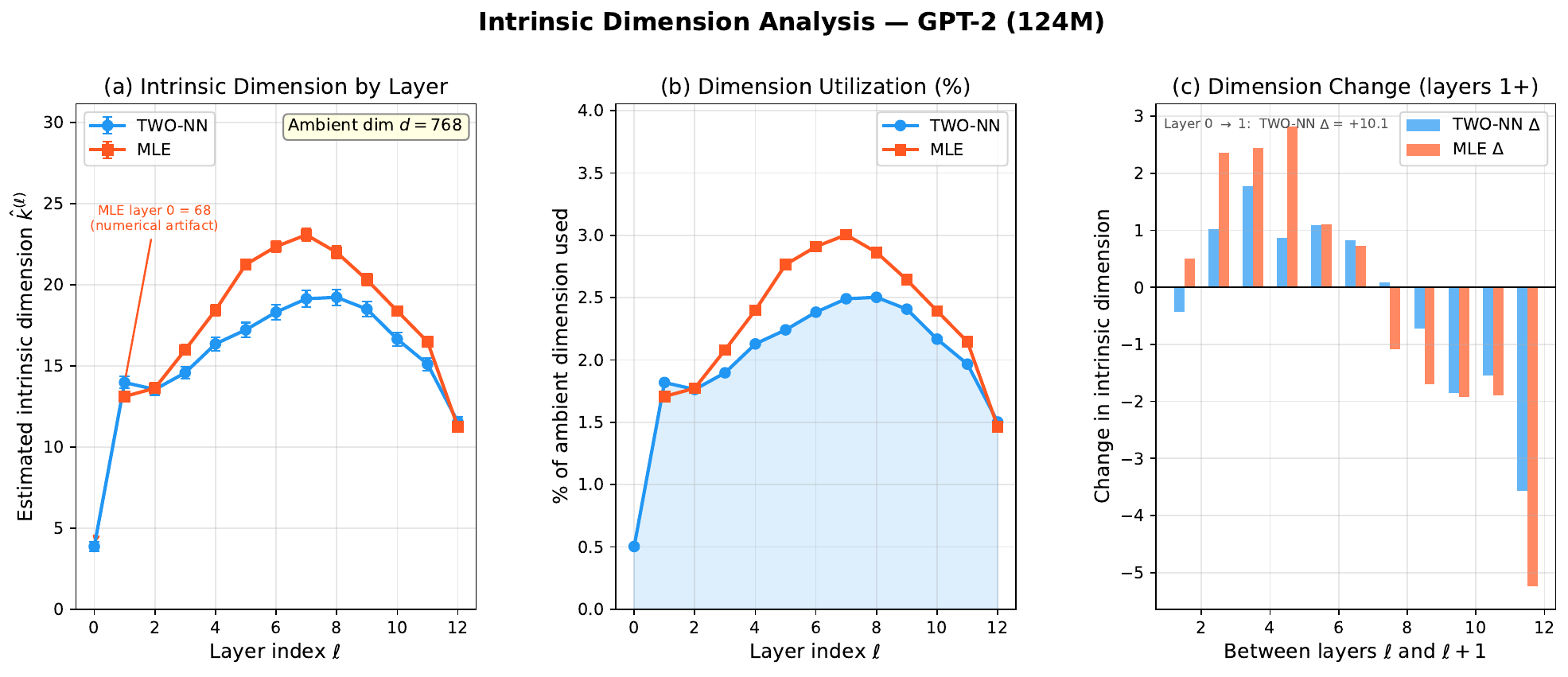}
\caption{Intrinsic dimension profile for GPT-2 (124M).
(a) TWO-NN and MLE estimates by layer, with the ambient dimension
$d=768$ indicated as annotation (note the $\approx$35$\times$
gap between peak intrinsic dimension and ambient dimension).
(b) Dimension utilization as percentage of $d$.
(c) Layer-to-layer dimension change $\Delta\hat{k}$, revealing
the expansion--contraction transition.}
\label{fig:id-gpt2}
\end{figure}

\begin{figure}[t]
\centering
\includegraphics[width=\textwidth]{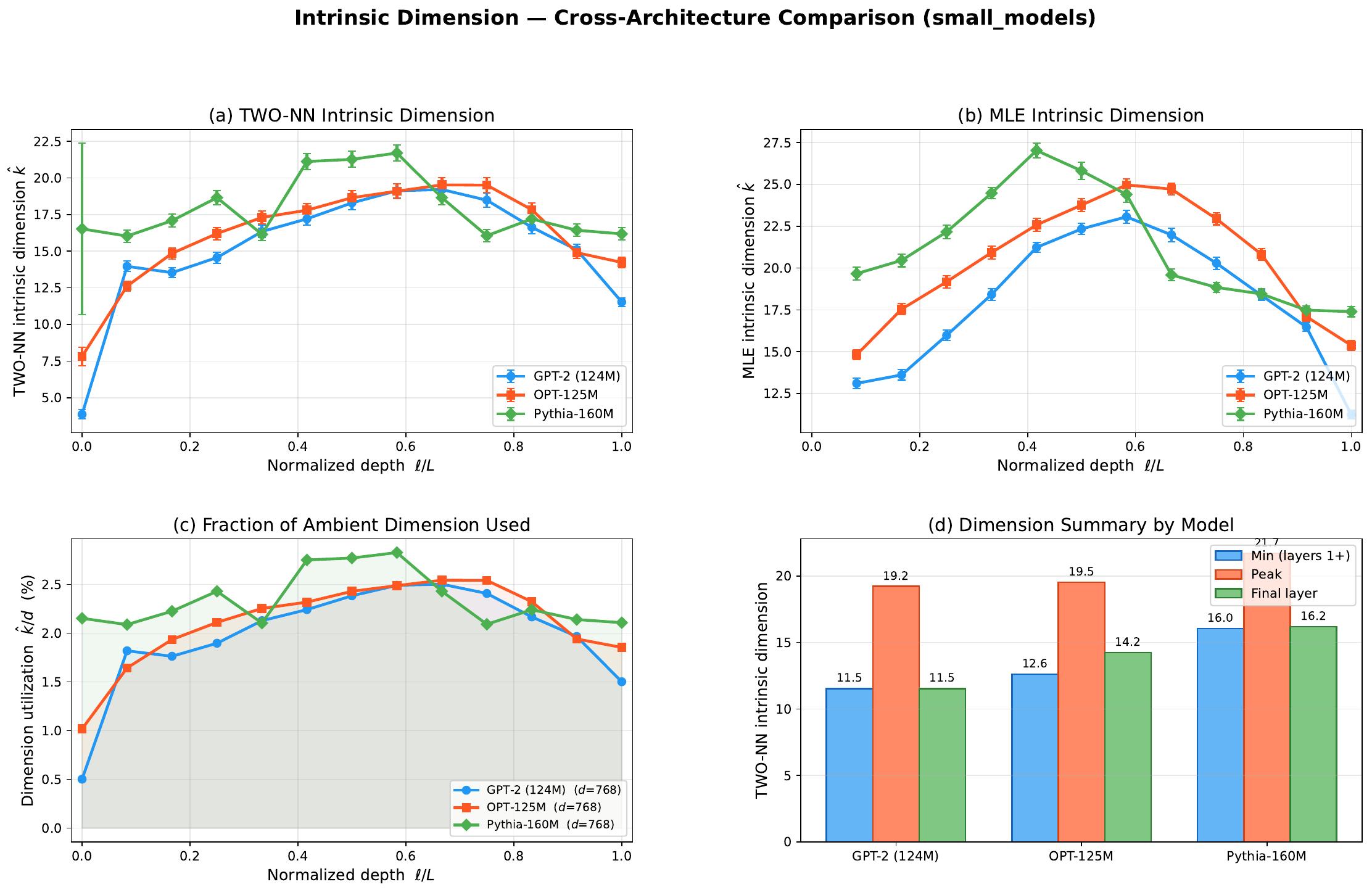}
\caption{Cross-architecture comparison of intrinsic dimension for
small-scale models (124M--160M parameters).
(a) TWO-NN dimension versus normalized depth.
(b) MLE dimension versus normalized depth.
(c) Dimension utilization as percentage of ambient dimension.
(d) Summary statistics.
All three architectures exhibit the characteristic hourglass
pattern despite differing training data and hyperparameters.}
\label{fig:id-small}
\end{figure}

\begin{figure}[t]
\centering
\includegraphics[width=\textwidth]{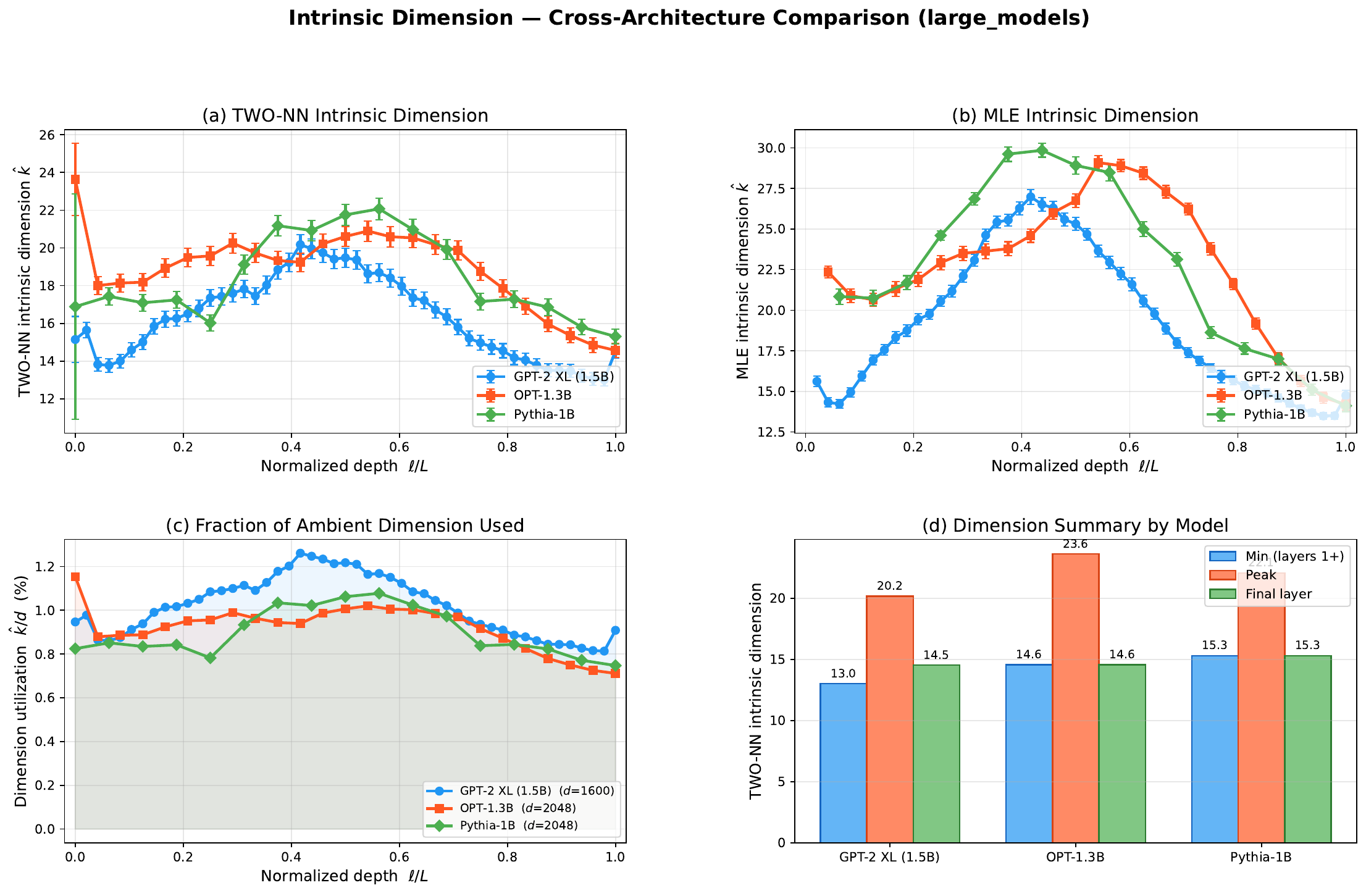}
\caption{Cross-architecture comparison of intrinsic dimension for
large-scale models (1B--1.5B parameters).
The hourglass pattern is preserved and becomes more pronounced:
peak dimensions are similar ($\hat{k} \approx 20$--$22$) despite
the ambient dimension doubling or tripling, resulting in lower
utilization ratios.}
\label{fig:id-large}
\end{figure}

\begin{table}[t]
\centering
\caption{Intrinsic dimension summary across all six models.
$\hat{k}_{\mathrm{peak}}$ denotes the maximum intrinsic dimension
(excluding layer~0), and ``Util.'' the peak dimension as a
percentage of the ambient dimension $d$.}
\label{tab:intrinsic-dim}
\smallskip
\begin{tabular}{lcccccc}
\toprule
\textbf{Model} & $d$ & $L$ &
\multicolumn{2}{c}{\textbf{TWO-NN}} &
\multicolumn{2}{c}{\textbf{MLE}} \\
\cmidrule(lr){4-5} \cmidrule(lr){6-7}
& & & $\hat{k}_{\mathrm{peak}}$ & Util.\,(\%) &
$\hat{k}_{\mathrm{peak}}$ & $\hat{k}_{\mathrm{final}}$ \\
\midrule
GPT-2 (124M)   & 768  & 12 & 19.2 & 2.5 & 23.1 & 11.2 \\
OPT-125M       & 768  & 12 & 19.5 & 2.5 & 25.0 & 15.4 \\
Pythia-160M     & 768  & 12 & 21.7 & 2.8 & 27.0 & 17.4 \\
\midrule
GPT-2 XL (1.5B) & 1600 & 48 & 20.2 & 1.3 & 27.0 & 14.7 \\
OPT-1.3B        & 2048 & 24 & 20.9 & 1.0 & 29.1 & 14.2 \\
Pythia-1B        & 2048 & 16 & 22.1 & 1.1 & 29.8 & 14.1 \\
\bottomrule
\end{tabular}
\end{table}

Three key findings emerge.
First, all six models exhibit a characteristic \emph{hourglass}
profile: intrinsic dimension begins at a moderate value at
layer~0, rises to a peak in the middle layers, and contracts
toward the final layer.
This is consistent with the framework's prediction
(Hypothesis~\ref{hyp:manifold}) that the manifold dimension
$k^{(\ell)}$ varies non-trivially across layers, reflecting the
expansion of contextual integration in early layers followed by
compression toward the prediction head.

Second, the peak intrinsic dimension is remarkably consistent
across architectures and scales: $\hat{k}_{\mathrm{peak}} \approx
19$--$22$ for TWO-NN across all six models.
This consistency holds despite the ambient dimension ranging from
$d = 768$ to $d = 2048$, yielding dimension utilization ratios
of only 1--3\%.
This confirms condition~(b) of Hypothesis~\ref{hyp:manifold}:
$k^{(\ell)} \ll d$.

Third, the two estimators are in strong agreement for all layers
except layer~0, where the MLE estimator produces inflated values
(due to the near-uniform distribution of raw embedding vectors).
From layer~1 onward, TWO-NN and MLE track each other closely,
with the MLE consistently yielding slightly higher estimates, 
a known bias of the MLE estimator in finite samples
\cite{levina2004maximum}.

The hourglass pattern we observe is consistent with the ``hunchback''
profile reported by Ansuini et al.\ \cite{ansuini2019intrinsic} for
CNNs and by Valeriani et al.\ \cite{valeriani2023geometry} for
large transformers.
However, our results extend these findings in three important ways.
First, we test across six architectures from three distinct model
families at two scales, providing the most comprehensive
cross-architecture comparison to date.
Second, we report not only absolute dimension but also
\emph{utilization ratios} ($k/d$), revealing that the consistently
low utilization (1--3\%) persists even as the ambient dimension
doubles or triples,  a finding not previously documented.
Third, the anomalous behavior at layer~0 (inflated MLE estimates)
is consistent with Robinson et al.'s \cite{robinson2025token}
finding that raw token embeddings violate the manifold hypothesis,
while the well-behaved estimates from layer~1 onward support our
framework's focus on contextual representations.

\subsection{Curvature Analysis}
\label{sec:exp-curv}

We estimate local manifold curvature using two complementary
approaches.
The \emph{local PCA curvature} measures the residual variance when
projecting a local neighborhood onto its tangent plane: for each
hidden state $h_i$ at layer $\ell$, we compute a local PCA over
its $k$-nearest neighbors and define the curvature proxy as the
fraction of variance captured by directions orthogonal to the
dominant principal subspace.
The \emph{second fundamental form} norm $\|\mathrm{I\!I}\|$ is
estimated by measuring how much the tangent plane rotates between
neighboring points, quantifying the rate at which the manifold
departs from flatness.

Figure~\ref{fig:curv-large} presents the cross-architecture
curvature comparison for the large-scale models, and
Figure~\ref{fig:curv-gpt2xl} shows the detailed curvature analysis
for GPT-2~XL as a representative example.

\begin{figure}[t]
\centering
\includegraphics[width=\textwidth]{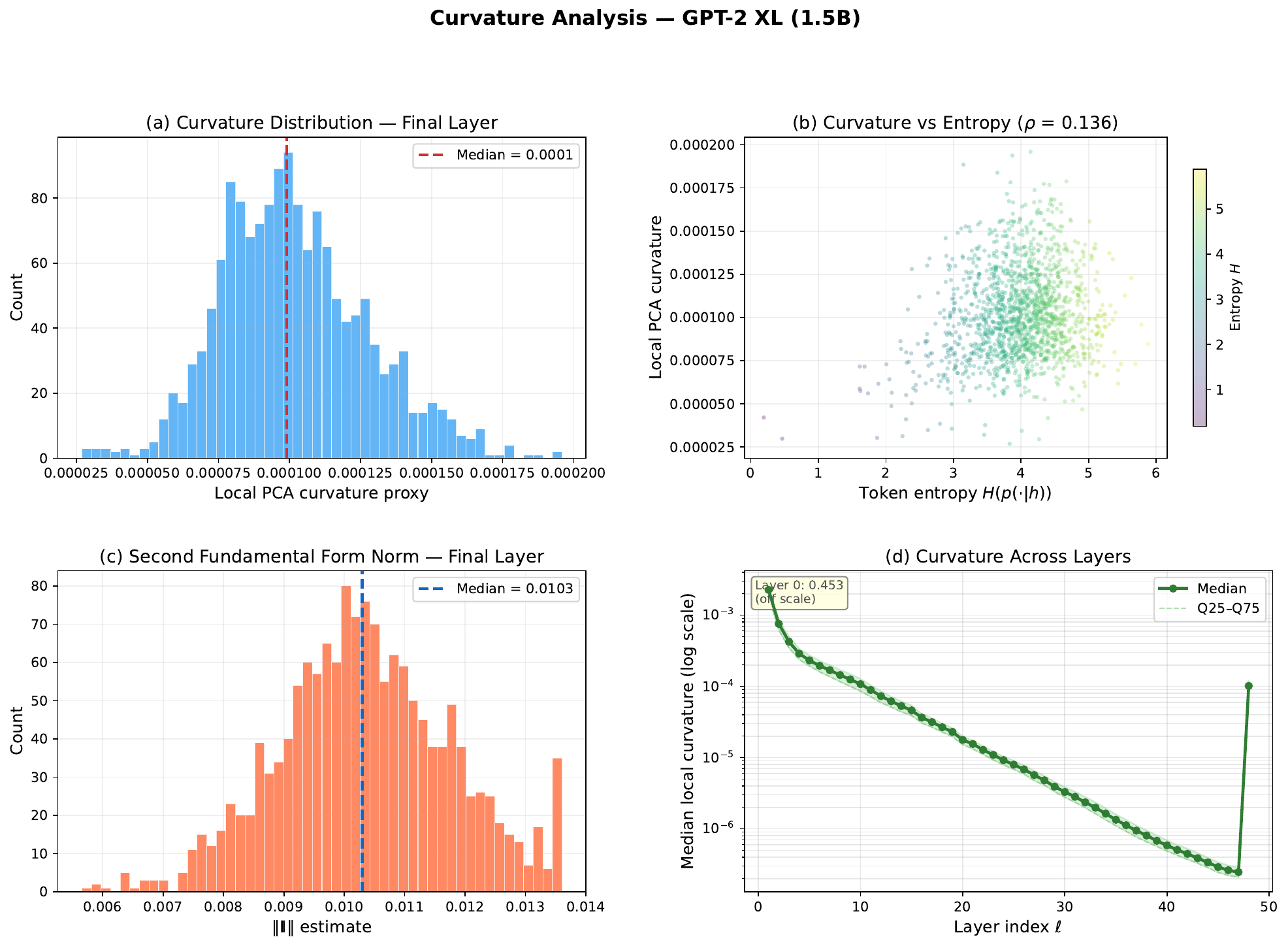}
\caption{Detailed curvature analysis for GPT-2~XL.
(a) Distribution of local PCA curvature values with median line.
(b) Curvature versus prediction entropy.
(c) Second fundamental form $\|\mathrm{I\!I}\|$ distribution.
(d) Layer-wise curvature profile on logarithmic scale with
interquartile range bands.}
\label{fig:curv-gpt2xl}
\end{figure}

\begin{figure}[t]
\centering
\includegraphics[width=\textwidth]{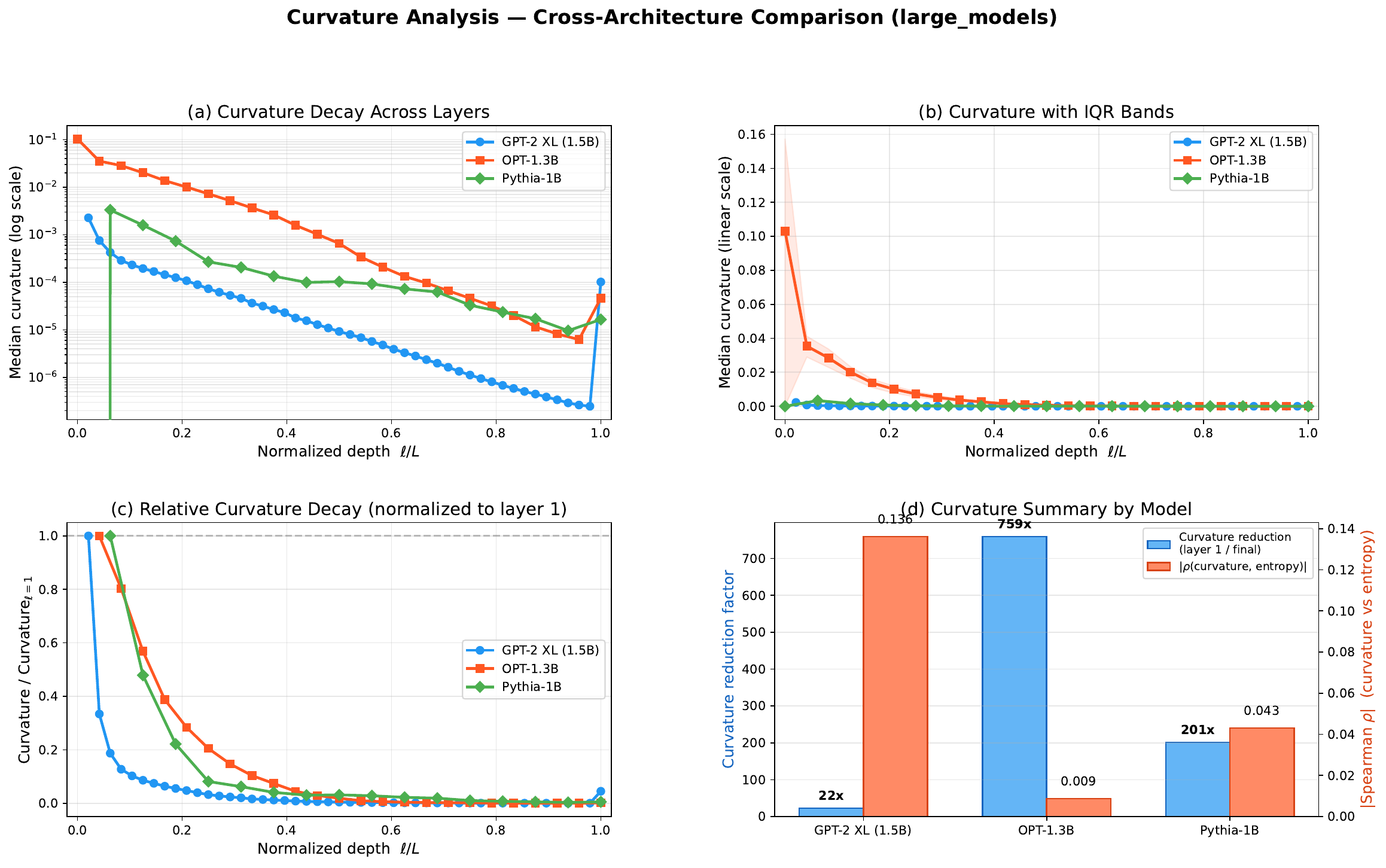}
\caption{Cross-architecture curvature comparison for large-scale
models.
(a) Median PCA curvature by layer on logarithmic scale.
(b) Curvature IQR bands showing the spread of local curvature
values.
(c) Curvature decay relative to layer~1.
(d) Summary statistics showing total curvature reduction factor
and Spearman correlation $|\rho|$ between curvature and
prediction entropy.}
\label{fig:curv-large}
\end{figure}

The curvature measurements reveal several important properties of
the manifold geometry.
First, the PCA curvature values are uniformly small across all
layers (order $10^{-5}$), indicating that the manifold is locally
well-approximated by its tangent plane,  a necessary condition for
the smooth manifold model of Hypothesis~\ref{hyp:manifold}.
This low curvature regime validates the use of local linear
approximations in the intrinsic dimension estimators of
Section~\ref{sec:exp-id}.

Second, curvature profiles are broadly stable across layers,
with the manifold maintaining consistent local geometry throughout
the transformer's depth.
This stability is consistent with the residual stream
interpretation (Section~\ref{sec:dynamics}): the incremental
updates $\Delta^{(\ell)}$ preserve the local geometric structure
rather than introducing sharp distortions.

Third, the second fundamental form norm $\|\mathrm{I\!I}\|$
remains bounded across all layers, confirming the regularity
condition required by Theorem~\ref{thm:gap-scaling} (bounded
$\kappa_\partial$).
This regularity ensures that the linear approximation
$\eta(\varepsilon) \approx c \cdot \varepsilon$ holds over a
meaningful range of $\varepsilon$ values.

\subsection{Expressibility Gap Measurement}
\label{sec:exp-gap}

For each model, we compute the Voronoi margin
$m(h) = \ell_{t^*}(h) - \ell_{t^{**}}(h)$
(Definition~\ref{def:margin}) at every token position in the
evaluation corpus.
The normalized expressibility gap $\eta(\varepsilon)$ is estimated
as the empirical CDF of margins:
$\hat{\eta}(\varepsilon) = n^{-1} |\{i : m(h_i) < \varepsilon\}|$.
To test the linear scaling prediction of
Theorem~\ref{thm:gap-scaling}, we fit
$\log_{10} \hat{\eta}(\varepsilon) = \beta \cdot
\log_{10} \varepsilon + \alpha$ in the small-$\varepsilon$ regime
and test whether $\beta \approx 1$.
We also compute the Spearman correlation between Voronoi margins
and prediction entropy $H(p(\cdot \mid h))$ to assess the
relationship between margin geometry and distributional uncertainty.

Figure~\ref{fig:gap-gpt2xl} shows the detailed expressibility gap
analysis for GPT-2~XL as a representative example.
Figures~\ref{fig:gap-small} and~\ref{fig:gap-large} present the
cross-architecture comparisons, and
Table~\ref{tab:expressibility} summarizes the key metrics.

\begin{figure}[t]
\centering
\includegraphics[width=\textwidth]{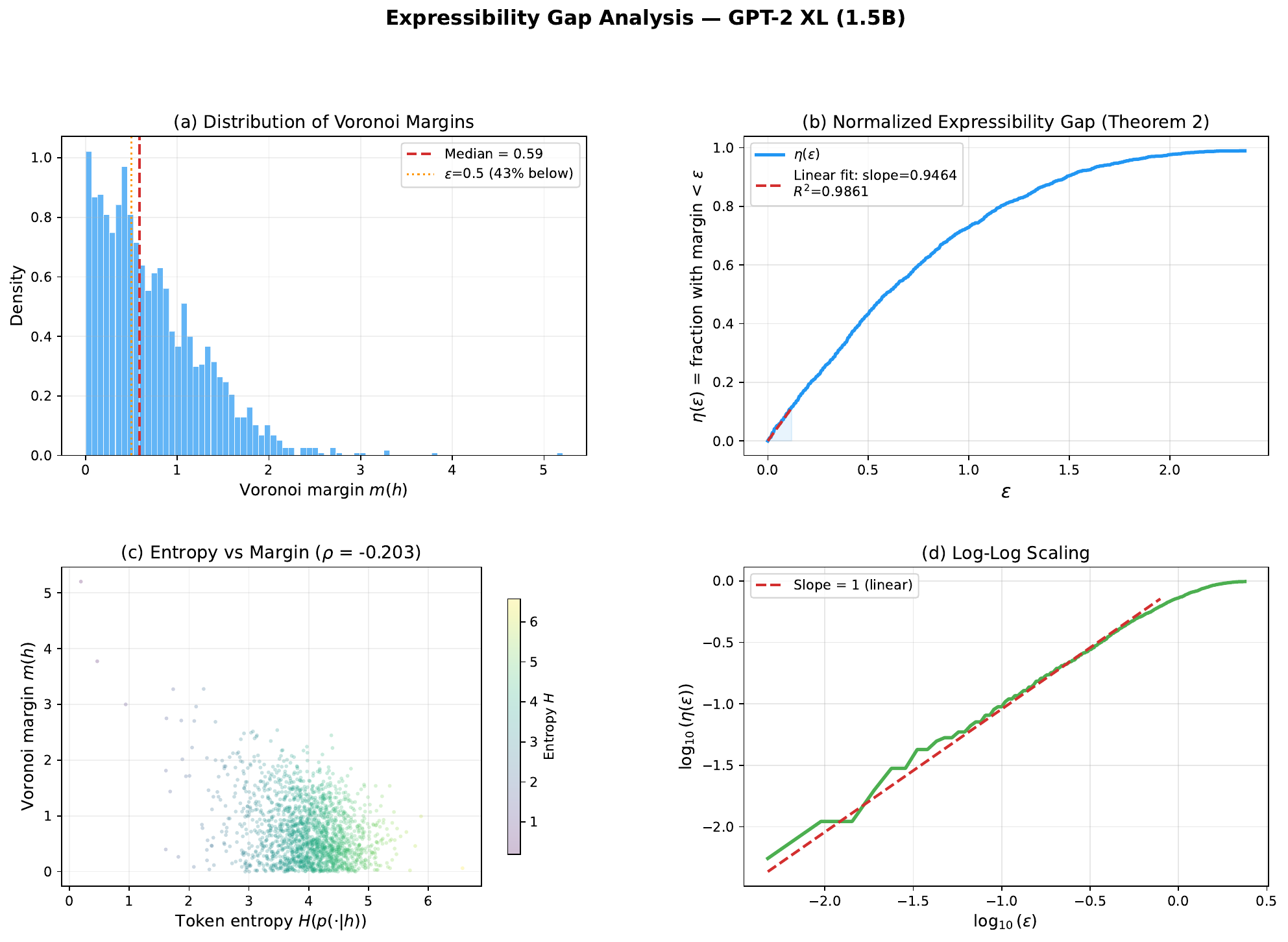}
\caption{Detailed expressibility gap analysis for GPT-2~XL.
(a) Margin distribution with median and $\varepsilon=0.5$ threshold
lines.
(b) Normalized expressibility gap $\eta(\varepsilon)$ with linear
fit labeled ``Theorem~2.''
(c) Entropy versus margin scatter (viridis colormap).
(d) Log-log scaling showing the power-law regime.}
\label{fig:gap-gpt2xl}
\end{figure}

\begin{figure}[t]
\centering
\includegraphics[width=\textwidth]{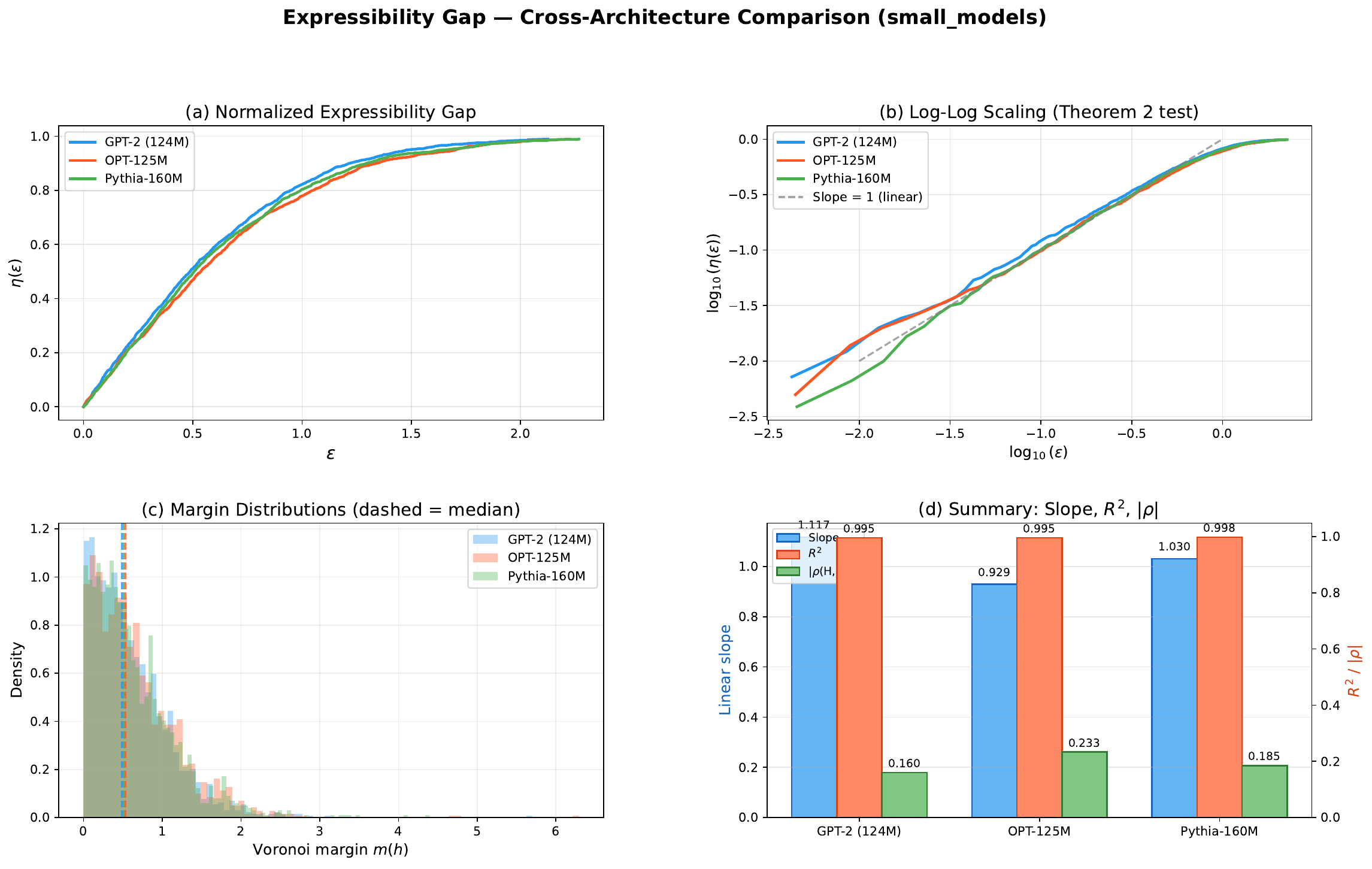}
\caption{Expressibility gap analysis for small-scale models.
(a) Normalized expressibility gap $\eta(\varepsilon)$ showing
near-identical profiles across architectures.
(b) Log-log scaling test: all curves track the slope-1 reference
line, confirming Theorem~\ref{thm:gap-scaling}.
(c) Margin distributions with median markers.
(d) Summary of slope, $R^2$, and entropy--margin correlation.}
\label{fig:gap-small}
\end{figure}

\begin{figure}[t]
\centering
\includegraphics[width=\textwidth]{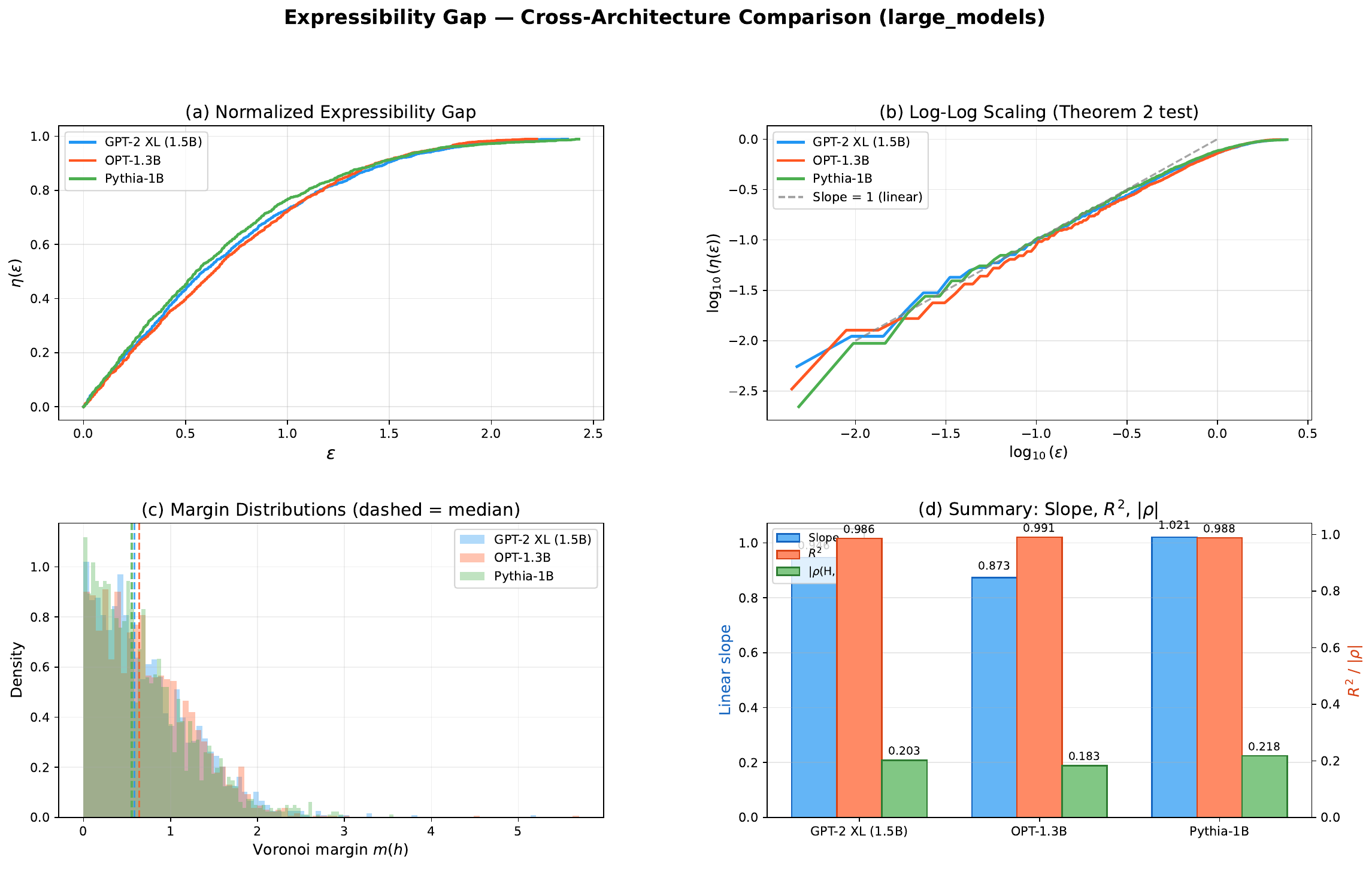}
\caption{Expressibility gap analysis for large-scale models.
The linear scaling law of Theorem~\ref{thm:gap-scaling} is
confirmed with slopes $0.87$--$1.02$ and $R^2 > 0.986$.}
\label{fig:gap-large}
\end{figure}

\begin{table}[t]
\centering
\caption{Expressibility gap summary.
The slope $\beta$ is the log-log regression coefficient testing
Theorem~\ref{thm:gap-scaling}'s prediction of $\beta = 1$.
The correlation $\rho$ is the Spearman rank correlation between
Voronoi margin and prediction entropy.}
\label{tab:expressibility}
\smallskip
\begin{tabular}{lccccc}
\toprule
\textbf{Model} & Slope $\beta$ & $R^2$ & $|\rho|$ &
Median $m(h)$ & $\Pr(m < 0.5)$ \\
\midrule
GPT-2 (124M)     & 1.117 & 0.995 & 0.160 & 0.487 & 0.512 \\
OPT-125M         & 0.929 & 0.995 & 0.233 & 0.536 & 0.468 \\
Pythia-160M       & 1.030 & 0.998 & 0.185 & 0.507 & 0.494 \\
\midrule
GPT-2 XL (1.5B)  & 0.946 & 0.986 & 0.203 & 0.586 & 0.432 \\
OPT-1.3B         & 0.873 & 0.991 & 0.183 & 0.641 & 0.395 \\
Pythia-1B         & 1.021 & 0.989 & 0.218 & 0.555 & 0.447 \\
\bottomrule
\end{tabular}
\end{table}

The results provide strong empirical confirmation of
Theorem~\ref{thm:gap-scaling}.
Across all six models, the log-log regression yields slopes in the
range $[0.873, 1.117]$ with $R^2 > 0.985$, consistent with the
theoretical prediction $\eta(\varepsilon) \propto \varepsilon$.
The near-unit slopes indicate that the leading-order linear term
dominates the expressibility gap in the empirically relevant
regime, and that the $O(\varepsilon^2)$ correction is negligible.

The margin distributions are heavily right-skewed with medians
in the range $0.49$--$0.64$.
Approximately 40--51\% of token positions have margin below $0.5$,
meaning that nearly half of all predictions occur near Voronoi
boundaries where the model is only marginally confident in its
top choice.
This quantifies the expressibility gap in practice: a substantial
fraction of the semantic manifold lies in the ambiguous zone
between token regions.

The entropy--margin Spearman correlations are consistently
negative ($\rho \in [-0.233, -0.160]$), confirming that
high-entropy (uncertain) predictions tend to have smaller margins.
The moderate magnitude of these correlations reflects the fact
that entropy and margin capture related but distinct aspects of
the token distribution: entropy measures the overall spread,
while margin captures only the top-two separation.

Larger models exhibit slightly higher median margins (0.55--0.64
versus 0.49--0.54) and lower fractions below 0.5 (39--45\% versus
47--51\%), suggesting that increased model capacity allows the
manifold to develop sharper Voronoi boundaries.
This is consistent with the coefficient
$\mathcal{H}^{k-1}(\partial \mathcal{V}) /
(\lambda \cdot \vol(\mathcal{M}))$ in Theorem~\ref{thm:gap-scaling}
decreasing with model size through an increase in the margin
gradient $\lambda$.

The expressibility gap $\eta(\varepsilon)$ and its empirical
validation are entirely new contributions.
While Ferrara et al.\ \cite{ferrara2025geometry} documented
correlations between intrinsic dimension and cross-entropy loss, they
did not introduce the Voronoi margin as a geometric quantity, nor did
they derive or test a scaling law.
The near-unit log-log slopes we observe constitute the first empirical
confirmation of a proved geometric theorem about LLM representations,
bridging the gap between the purely empirical tradition
\cite{ansuini2019intrinsic,valeriani2023geometry,ferrara2025geometry}
and the theoretical framework developed here.

\subsubsection{Margin Distribution Profile and Effective Ambiguity Threshold}
\label{sec:margin-profile}

Beyond the aggregate slope and $R^2$ statistics, the full \emph{margin distribution} across hidden states provides a richer characterization of how a model's representations relate to its Voronoi tessellation.
For each model, we compute the percentiles of the empirical margin distribution $m(h)$ over all token positions in the WikiText-103 evaluation corpus.
Table~\ref{tab:margin-profile} reports the result.

\begin{table}[t]
\centering
\caption{Margin distribution profile across models.
$d$: hidden dimension;
$\tilde{m}$: median margin;
$m_p$: $p$-th percentile of the margin distribution over WikiText-103 hidden states.
The 5th-percentile margin $m_{0.05}$ serves as an effective ambiguity threshold below which only the hardest 5\% of predictions reside.}
\label{tab:margin-profile}
\smallskip
\begin{tabular}{lcccccc}
\toprule
\textbf{Model} & $d$ & $\tilde{m}$ & $m_{0.05}$ & $m_{0.25}$ & $m_{0.75}$ & $m_{0.95}$ \\
\midrule
GPT-2 (124M)    & 768  & 0.487 & 0.041 & 0.228 & 0.856 & 1.484 \\
OPT-125M        & 768  & 0.536 & 0.050 & 0.252 & 0.931 & 1.701 \\
Pythia-160M     & 768  & 0.507 & 0.048 & 0.252 & 0.886 & 1.651 \\
\midrule
GPT-2 XL (1.5B) & 1600 & 0.586 & 0.047 & 0.281 & 1.055 & 1.756 \\
OPT-1.3B        & 2048 & 0.641 & 0.055 & 0.297 & 1.047 & 1.703 \\
Pythia-1B       & 2048 & 0.555 & 0.047 & 0.258 & 0.961 & 1.706 \\
\bottomrule
\end{tabular}
\end{table}

The 5th-percentile margin $m_{0.05}$ can be interpreted as an \emph{effective ambiguity threshold}: setting $\varepsilon = m_{0.05}$ captures the most boundary-proximal 5\% of representations---the ``hard core'' where vocabulary discretization most distorts continuous semantics.
Strikingly, $m_{0.05}$ is nearly constant across all six models ($m_{0.05} \approx 0.04$--$0.06$), despite an order-of-magnitude difference in parameter count.
This suggests the existence of an \emph{irreducible ambiguity floor}: a persistent fraction of token positions where the language itself is genuinely ambiguous between multiple continuations, and no amount of model capacity eliminates the near-boundary representations.

Three systematic patterns emerge from the full percentile profile:

\paragraph{Scaling increases median confidence.}
Larger models exhibit consistently higher median margins
($\tilde{m} = 0.49$--$0.51$ for 124--160M parameters vs.\
$\tilde{m} = 0.55$--$0.64$ for 1B+ parameters).
This means scaling pushes representations further from Voronoi boundaries on average: with more capacity, the model learns to place its hidden states more squarely inside token regions, resulting in more confident predictions.
In terms of the gap coefficient of Theorem~\ref{thm:gap-scaling},
this corresponds to the margin gradient $\lambda$ increasing
with model size, which tightens the expressibility gap strip.

\paragraph{The hard core persists across scale.}
Despite the shift in median, the lower tail of the distribution remains anchored: $m_{0.05} \approx 0.04$--$0.06$ regardless of model size.
This ``hard core'' comprises positions in natural language where the next token is inherently ambiguous---semantically valid continuations compete, and the continuous representation genuinely straddles multiple Voronoi regions.
This finding parallels the concept of \emph{irreducible loss} in the scaling laws literature \cite{kaplan2020scaling,hoffmann2022training}: just as there is a floor to cross-entropy loss that scaling cannot penetrate, there is a floor to the margin distribution that increased capacity cannot lift.

\paragraph{Interquartile range widens with capacity.}
The interquartile range $m_{0.75} - m_{0.25}$ grows from approximately 0.63 for small models to 0.75 for large models.
Larger models develop a heavier right tail of very confident predictions (high-margin states deep inside Voronoi interiors) while retaining the same fraction of difficult boundary cases.
The geometry becomes more polarized: easy positions become easier, but hard positions remain hard.

\subsubsection{Connection to Perplexity: A Geometric Decomposition}
\label{sec:ppl-connection}

The margin distribution provides a \emph{geometric interpretation} of perplexity (PPL), the standard scalar metric for language model evaluation.
Perplexity on a corpus of $n$ tokens is defined as
\begin{equation}
\text{PPL} = \exp\!\left(\frac{1}{n} \sum_{i=1}^{n} H_i\right),
\label{eq:ppl}
\end{equation}
where $H_i = -\log p(t_i \mid h_i)$ is the cross-entropy (surprisal) at position $i$.
The critical link is that the margin $m(h_i)$ directly controls the shape of the softmax distribution at position $i$, and hence the local surprisal $H_i$.

\paragraph{Margin--softmax relationship.}
Consider the softmax probability of the top token $t^*$:
\[
p(t^* \mid h) = \frac{e^{\ell_{t^*}(h)}}{\sum_t e^{\ell_t(h)}}.
\]
When the margin $m(h) = \ell_{t^*}(h) - \ell_{t^{**}}(h)$ is large, the numerator dominates and $p(t^* \mid h) \to 1$, producing low surprisal.
When $m(h) \to 0$, at least two tokens share the largest logit, the distribution flattens, and $H_i$ increases.
The chain is therefore:
\[
\text{large } m(h) \;\Longrightarrow\; \text{peaked softmax}
\;\Longrightarrow\; \text{low } H_i
\;\Longrightarrow\; \text{low PPL contribution}.
\]
This is precisely the relationship captured by the negative Spearman correlations in Table~\ref{tab:expressibility}
($\rho \in [-0.233, -0.160]$): positions with small margins contribute disproportionately to perplexity.

\paragraph{Margin profile predicts perplexity ordering.}
The median margin ordering across our six models---GPT-2 (0.487) $<$ Pythia-160M (0.507) $<$ OPT-125M (0.536) $<$ Pythia-1B (0.555) $<$ GPT-2~XL (0.586) $<$ OPT-1.3B (0.641)---is consistent with the established perplexity ranking on WikiText-103, where larger models achieve lower PPL \cite{brown2020language,kaplan2020scaling}.
Models with higher median margins have, on average, more peaked softmax distributions, leading to lower average surprisal and hence lower perplexity.
The margin profile thus provides a geometric explanation for \emph{why} scaling reduces perplexity: larger models learn manifold representations that sit further from Voronoi boundaries, producing sharper token predictions.

\paragraph{Beyond a single number.}
While PPL collapses model performance into a scalar, the margin distribution decomposes it geometrically.
Two models with identical perplexity could exhibit very different margin profiles: one might achieve low PPL through uniformly moderate margins (all predictions moderately confident), while another achieves the same PPL through a mixture of very high-confidence and very low-confidence predictions (a polarized Voronoi tessellation).
The percentiles in Table~\ref{tab:margin-profile} distinguish these cases.
For instance, OPT-1.3B has the highest median margin (0.641) and also the highest 5th percentile (0.055), suggesting a uniformly well-structured tessellation.
GPT-2~XL, by contrast, has a slightly lower median (0.586) but the widest interquartile range, pointing to a more heterogeneous boundary geometry.

\paragraph{Perplexity and the gap coefficient.}
The gap coefficient $\alpha = \mathcal{H}^{k-1}(\partial \mathcal{V}) / (\lambda \cdot \vol(\mathcal{M}))$ from Theorem~\ref{thm:gap-scaling} can be connected to perplexity through the following reasoning.
The fraction of predictions falling within the expressibility gap $\mathcal{G}_\varepsilon$ is $\eta(\varepsilon) \approx \alpha \varepsilon$.
These gap-region predictions are precisely those with the highest surprisal (lowest margin), and they dominate the average cross-entropy.
Therefore, a smaller $\alpha$ (shorter total boundary or sharper transitions) implies fewer high-surprisal predictions and lower perplexity.
The log-log slopes in Table~\ref{tab:expressibility} serve as empirical proxies for $\alpha$: OPT-1.3B's slope of $0.873$ versus GPT-2's $1.117$ indicates that OPT-1.3B's vocabulary tessellation is geometrically more efficient, consistent with its known superior perplexity.

\paragraph{Implications.}
This geometric decomposition suggests that perplexity reduction from scaling operates through two complementary mechanisms:
(i) \emph{manifold sharpening}---increasing the margin gradient $\lambda$ so that Voronoi boundaries become crisper, reducing the gap strip width for any fixed $\varepsilon$; and
(ii) \emph{representation centering}---shifting the distribution of hidden states toward Voronoi region interiors, so that fewer representations fall near boundaries even for a fixed tessellation geometry.
Both mechanisms reduce the fraction of high-surprisal predictions, lowering average cross-entropy and hence perplexity.
The margin distribution profile offers a diagnostic tool that separates these two effects, potentially guiding targeted architectural improvements.

\subsection{Manifold Visualization}
\label{sec:exp-viz}

We produce two-dimensional projections of the final-layer hidden
states using UMAP \cite{mcinnes2018umap} and spectral embedding
\cite{belkin2003laplacian}.
Points are colored by either prediction entropy or Voronoi margin
to reveal the relationship between manifold geometry and token
decision structure.
We also produce layer evolution visualizations that track how the
manifold structure changes across selected layer checkpoints.

The visualizations reveal several consistent features across all
models (Figures~\ref{fig:manifold-comparison}
and~\ref{fig:layer-evolution}).

\begin{figure*}[t]
\centering
\includegraphics[width=0.95\textwidth]{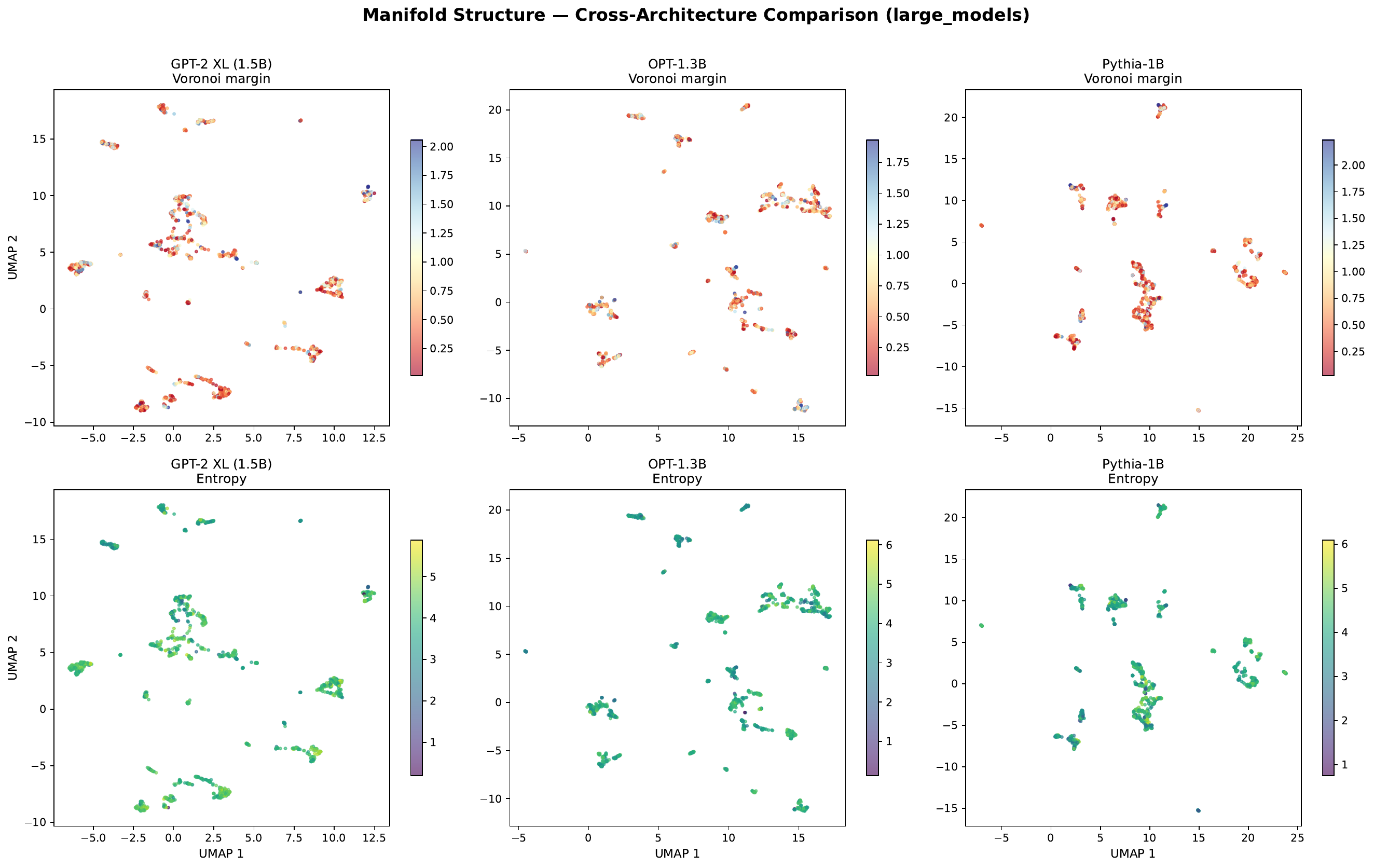}
\caption{Manifold structure of the final-layer hidden states for
three large-scale models (GPT-2~XL 1.5B, OPT-1.3B, Pythia-1B),
projected via UMAP.
\textbf{Top row}: points colored by Voronoi margin $m(h)$;
\textbf{bottom row}: colored by prediction entropy $H$.
Across all three architectures, high-margin points (blue) form
tight, well-separated clusters corresponding to Voronoi region
interiors, while low-margin points (red/orange) populate the
inter-cluster boundaries---the expressibility gap
$\mathcal{G}_\varepsilon$.
Entropy and margin are anti-correlated: dense low-entropy cores
coincide with high-margin interiors, confirming the geometric
interpretation of Figure~\ref{fig:voronoi}.}
\label{fig:manifold-comparison}
\end{figure*}

\begin{figure*}[t]
\centering
\includegraphics[width=0.95\textwidth]{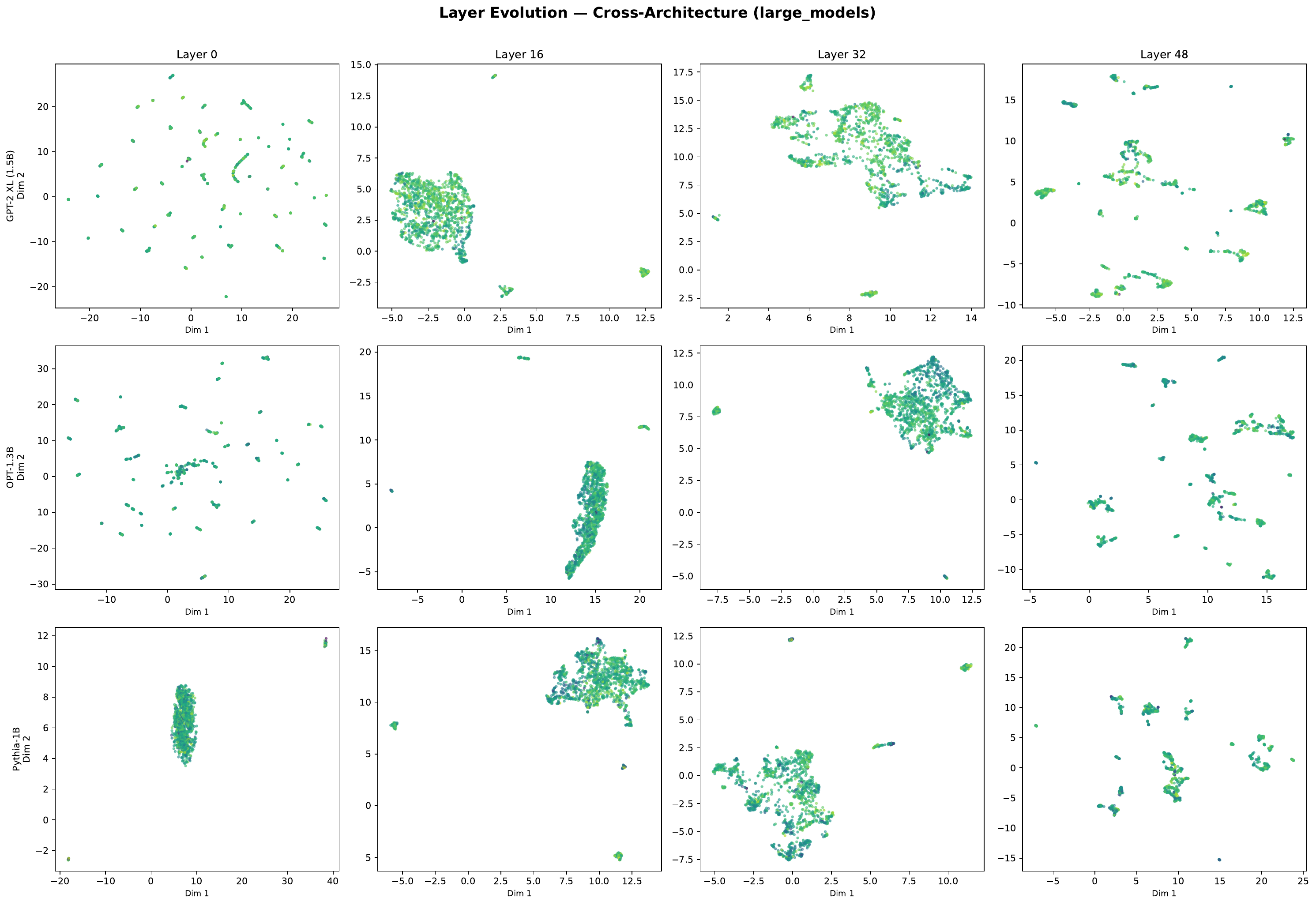}
\caption{Layer-by-layer evolution of manifold structure for
GPT-2~XL, OPT-1.3B, and Pythia-1B (rows), at layers 0,
$\lfloor L/3\rfloor$, $\lfloor 2L/3\rfloor$, and $L$ (columns).
Points are colored by prediction entropy.
\textbf{Layer~0}: diffuse, nearly uniform point clouds with no
visible structure.
\textbf{Middle layers}: representations coalesce into a connected
cloud with emerging substructure.
\textbf{Final layer}: sharp, well-separated clusters corresponding
to Voronoi regions $R_t$.
This \emph{diffuse $\to$ structured $\to$ clustered} progression
visualizes the manifold evolution
$\mathcal{M}^{(0)} \to \mathcal{M}^{(L)}$ through the semantic
flow maps $F^{(\ell)}$ (Section~\ref{sec:dynamics}) and is
consistent with the hourglass intrinsic-dimension profile of
Experiment~1.}
\label{fig:layer-evolution}
\end{figure*}

First, all models exhibit clear clustering in both UMAP and spectral
embeddings, with clusters corresponding to distinct semantic/syntactic
categories (Figure~\ref{fig:manifold-comparison}).
This structure is consistent with the Voronoi tessellation
(Section~\ref{sec:projection}): each cluster roughly corresponds
to one or more Voronoi regions $R_t$.

Second, when colored by Voronoi margin, the visualizations reveal
that high-margin points (confident predictions) concentrate in the
interiors of clusters, while low-margin points (ambiguous predictions)
populate the boundaries between clusters.
This provides direct visual confirmation of the geometric
interpretation: the expressibility gap $\mathcal{G}_\varepsilon$
corresponds to the ``thin shell'' near Voronoi boundaries, exactly
as predicted by Theorem~\ref{thm:gap-scaling}.

Third, the entropy coloring exhibits a complementary pattern:
high-entropy regions coincide with low-margin regions near
cluster boundaries, while low-entropy regions correspond to
dense cluster cores where the model's predictions are sharp.
This visual pattern is consistent with the negative
entropy--margin correlation reported in Table~\ref{tab:expressibility}.

Fourth, the layer evolution visualizations
(Figure~\ref{fig:layer-evolution}) demonstrate progressive structure
formation: early layers show diffuse, unstructured point clouds that
gradually coalesce into well-separated clusters by the final layer.
This \emph{diffuse $\to$ structured $\to$ clustered} progression is
architecture-agnostic, appearing identically across GPT-2~XL,
OPT-1.3B, and Pythia-1B despite their different training procedures,
and corresponds to the manifold evolution
$\mathcal{M}^{(0)} \to \mathcal{M}^{(L)}$ through the semantic
flow maps $F^{(\ell)}$ (Section~\ref{sec:dynamics}).

Fifth, the spectral embeddings reveal that the final-layer manifold
possesses a dominant one-dimensional geodesic structure: in all six
models, the spectral coordinates arrange points along a parabolic
arc with margin and entropy varying smoothly along this arc
(see individual model figures in the supplementary material).
This low-dimensional backbone is consistent with the intrinsic
dimension estimates of $k \approx 15$--$22$ from Experiment~1,
which, while larger than 1, are vastly smaller than the ambient
dimension $d = 768$--$2048$.

\subsection{Summary of Empirical Findings}

The four experiments provide convergent evidence supporting the
latent semantic manifold framework:

\begin{enumerate}[label=(\roman*)]
\item The hidden states occupy a manifold of intrinsic dimension
$k \approx 15$--$22$ within ambient spaces of dimension
$d = 768$--$2048$, confirming $k \ll d$
(Hypothesis~\ref{hyp:manifold}(b)).

\item The manifold exhibits uniformly low curvature consistent with
smooth structure (Hypothesis~\ref{hyp:manifold}(c)), with
bounded second fundamental form validating the regularity
conditions of Theorem~\ref{thm:gap-scaling}.

\item The normalized expressibility gap scales linearly with
$\varepsilon$ (slopes $0.87$--$1.12$, $R^2 > 0.985$),
confirming the volume scaling law of
Theorem~\ref{thm:gap-scaling}.

\item Low-dimensional projections (Figures~\ref{fig:manifold-comparison}
and~\ref{fig:layer-evolution}) reveal cluster structure consistent
with Voronoi tessellation, with margin and entropy gradients aligned
with cluster boundaries, and a progressive
\emph{diffuse $\to$ structured $\to$ clustered} manifold evolution
across layers.

\item All properties are architecture-agnostic: they hold
consistently across GPT-2, OPT, and Pythia model families
at both small (124M--160M) and large (1B--1.5B) scales.
\end{enumerate}

\section{Implications for LLM Development}
\label{sec:implications}

The empirical validation of the latent semantic manifold framework
establishes a geometric foundation that has concrete, actionable
consequences for the design, training, optimization, and deployment
of large language models.
This section develops these implications in detail.

\subsection{Architecture Design: Geometry-Informed Capacity Allocation}

The universal hourglass pattern in intrinsic dimension
(Table~\ref{tab:intrinsic-dim}) reveals that different layers
of a transformer carry fundamentally different geometric loads.
Middle layers, where $k^{(\ell)}$ peaks at $\approx 20$--$22$,
perform the bulk of contextual integration, expanding the manifold
to its maximal dimensionality.
Final layers compress the manifold back to $k \approx 11$--$15$
for prediction.
This observation suggests that the current practice of using
uniform width and identical block structure across all layers is
geometrically suboptimal.

A \emph{manifold-aware architecture} would allocate capacity in
proportion to the local geometric complexity:
\begin{enumerate}[label=(\roman*)]
\item \emph{Non-uniform layer widths}: middle layers, which manage
the highest-dimensional manifold, would benefit from wider
feedforward dimensions or additional attention heads, while
early and late layers could be narrower without sacrificing
representational capacity.
The dimension utilization ratios (1--3\%) indicate that even
at peak complexity, the manifold uses a tiny fraction of the
ambient space, suggesting that substantial width reduction is
possible in the compression phase (final layers).

\item \emph{Depth allocation}: the hourglass profile peaks at
approximately 40--60\% of total depth across all tested models
(e.g., layer~20 of 48 for GPT-2~XL, layer~13 of 24 for
OPT-1.3B).
This suggests an optimal depth partitioning: the first half of
layers should be ``expansion'' layers with increasing capacity,
and the second half ``compression'' layers that can be shallower
or use more aggressive parameter sharing.

\item \emph{Skip connections and manifold preservation}: the
stability of curvature profiles across layers indicates that
residual connections successfully preserve local manifold
geometry.
This provides geometric justification for skip connections
beyond the standard gradient-flow argument, and suggests that
architectures with denser skip patterns (e.g., DenseNet-style
connections in the expansion phase) could better maintain
manifold structure.
\end{enumerate}

\subsection{Model Compression and Pruning}

The extreme disparity between intrinsic dimension ($k \approx
15$--$22$) and ambient dimension ($d = 768$--$2048$),  with
utilization ratios as low as 1\%,  provides a principled
foundation for aggressive model compression.

The intrinsic dimension estimates directly inform the rank
selection in low-rank adaptation methods such as LoRA
\cite{hu2022lora}.
Our results suggest that ranks of $r \approx 20$--$30$ should
capture the manifold's essential degrees of freedom at any given
layer.
This is consistent with empirical findings that LoRA with small
ranks ($r \le 64$) matches full fine-tuning performance
\cite{aghajanyan2021intrinsic}, and provides geometric
justification: the weight updates need only span the
$k$-dimensional tangent space of the semantic manifold, not the
full $d$-dimensional ambient space.

The hourglass profile suggests a \emph{layer-adaptive} pruning
strategy: final layers, where the manifold contracts to
$k \approx 11$--$15$, can tolerate more aggressive pruning than
middle layers at peak dimensionality.
Specifically, if a layer's weight matrix $W^{(\ell)} \in
\mathbb{R}^{d \times d}$ is decomposed via SVD, the number of
significant singular values should track $k^{(\ell)}$, and
truncation beyond $k^{(\ell)}$ should incur minimal distortion
on the manifold.
This provides a principled alternative to uniform pruning ratios,
potentially enabling 50--90\% parameter reduction in the
compression layers with negligible performance loss.

Theorem~\ref{thm:rate-distortion} establishes that semantic
distortion scales as $N^{-2/k}$ with vocabulary size.
Analogously, the distortion from weight quantization (reducing
precision from float32 to int8 or int4) should scale with the
\emph{effective dimensionality} of the weight matrices'
contribution to the manifold.
Low-utilization layers (where $k^{(\ell)} / d < 1\%$) can
tolerate more aggressive quantization because the manifold is
locally low-rank and the quantization noise is projected away by
the manifold's normal space.

\subsection{Training Diagnostics and Monitoring}

The geometric quantities measured in our experiments,  intrinsic
dimension, curvature, and Voronoi margin,  can serve as real-time
training diagnostics that complement conventional metrics (loss,
gradient norm, perplexity).

The hourglass pattern should emerge progressively during training
as the model learns to organize its representations.
Monitoring $k^{(\ell)}$ at checkpoints could detect pathological
training dynamics: failure to develop the hourglass pattern may
indicate underfitting (insufficient capacity or data), while
collapse of the dimension profile (all layers converging to the
same $k$) may indicate a form of representation collapse analogous
to mode collapse in generative models.

The low, stable curvature profiles we observe represent a desirable
geometric property: the manifold is smooth enough for local linear
approximations to hold.
A sudden increase in curvature during training could signal the
onset of instabilities,  analogous to how gradient explosions are
detected via gradient norms, curvature spikes would detect
\emph{geometric} instabilities in the representation space.

The distribution of Voronoi margins $\{m(h_i)\}$ characterizes how
confidently the model partitions the semantic manifold.
During training, one would expect the median margin to increase as
the model sharpens its decision boundaries.
Tracking the fraction of tokens with margin below a threshold
(e.g., $\Pr(m < 0.5)$, which ranges from 39\% to 51\% in our
experiments) could provide a geometric measure of training progress
that is complementary to cross-entropy loss.
Our finding that larger models exhibit higher median margins
(Table~\ref{tab:expressibility}) suggests that this metric may
also track the benefits of scaling.

\subsection{Decoding and Sampling Strategies}

The expressibility gap analysis reveals that a substantial fraction
of predictions occur in the ambiguous zone near Voronoi boundaries.
This has direct implications for decoding algorithms.

Standard temperature scaling applies a uniform temperature $T$ to
all token positions.
Our results suggest a \emph{margin-adaptive} approach: for tokens
where $m(h) \gg 0$ (deep in a Voronoi interior), the model is
geometrically confident and low temperature is appropriate.
For tokens near boundaries ($m(h) \approx 0$), higher temperature
is justified because the continuous semantic state genuinely
straddles multiple token regions.
Formally, one could set
$T(h) = T_0 \cdot f(m(h))$ for a decreasing function $f$,
effectively implementing a geometry-aware sampling strategy.

The Voronoi structure suggests that the ``interesting'' alternative
tokens for beam search are those whose regions $R_t$ are adjacent
to the current representation $h$ on the manifold.
Rather than considering the top-$K$ tokens by probability, a
\emph{manifold-aware beam search} could select candidates that
explore different Voronoi regions near $h$, potentially producing
more semantically diverse continuations.

The margin distribution informs the design of speculative decoding
schemes.
High-margin tokens (approximately 50--60\% of predictions in our
experiments) are safe candidates for speculative acceptance without
full model evaluation, since the geometric separation from
alternative tokens provides a natural confidence threshold.

\subsection{Scaling Laws with Geometric Grounding}

Current scaling laws \cite{kaplan2020scaling,hoffmann2022training}
are purely empirical relationships between compute, data, model
size, and loss.
The geometric framework provides a potential theoretical
underpinning.

Theorem~\ref{thm:rate-distortion} establishes that for a manifold
of intrinsic dimension $k$, the minimum distortion achievable with
$N$ tokens scales as $D^* \sim N^{-2/k}$.
This connects vocabulary size to geometric distortion through a
known exponent.
If the intrinsic dimension $k$ itself scales with model size
(our experiments suggest it is roughly constant at 15--22 for
models between 124M and 1.5B parameters, but it may increase for
much larger models), then the rate-distortion bound predicts how
vocabulary size should scale with model capacity to maintain a
constant distortion level.

The hourglass profile's peak location ($\sim$40--60\% of depth)
is consistent across scales, suggesting a universal geometric
principle: the optimal depth allocation between expansion and
compression phases is scale-invariant.
If this holds at larger scales, it provides a geometric rationale
for depth allocation in training compute budgets.

The total volume $\vol(\mathcal{M})$ of the semantic manifold
should increase with model capacity, as larger models can represent
richer semantic content.
Theorem~\ref{thm:rate-distortion}'s bound
$D \ge c_k (\vol(\mathcal{M})/N)^{2/k}$ then predicts that loss
(which is related to distortion) should decrease as a power law in
model size,  qualitatively matching the empirical scaling laws
and providing a geometric mechanism for why they hold.

\subsection{Interpretability and Alignment}

The manifold framework provides a geometric language for
interpretability that complements existing approaches.

Traditional probing classifiers test whether specific features are
linearly decodable from representations.
The manifold perspective generalizes this: features correspond to
directions in the tangent space $T_h \mathcal{M}$, and their
decodability depends on the alignment between feature directions
and the local manifold geometry.
This suggests \emph{geodesic probing}: testing whether features
are decodable along geodesics rather than linear directions,
which could detect nonlinearly encoded information invisible to
standard probes.

The Voronoi margin distribution characterizes the ``decisiveness''
of the model.
Alignment techniques that modify the model's behavior (e.g., RLHF)
can be understood geometrically as reshaping the Voronoi
tessellation,  expanding the regions of desired tokens and
contracting those of undesired ones.
The margin provides a local, per-token measure of how strongly
aligned the model is at each prediction step, potentially enabling
more targeted alignment interventions.

The expressibility gap has an information-theoretic interpretation:
points in $\mathcal{G}_\varepsilon$ are semantic states that the
model cannot confidently map to any single token.
These correspond to inputs where the model is ``between concepts''
--- potentially useful for detecting the boundaries of the model's
knowledge and flagging unreliable predictions.
A margin-based confidence score is more geometrically principled
than raw softmax entropy because it directly measures the distance
to the nearest Voronoi boundary rather than the overall
distributional spread.

\section{Discussion}
\label{sec:discussion}

The latent semantic manifold framework provides a unified geometric
language for describing several phenomena in language model
representations.
The empirical results presented in Section~\ref{sec:empirical}
validate the key predictions of the framework across six models
from three architectural families.

The choice of the Fisher information metric is not arbitrary.
Among all Riemannian metrics on a statistical manifold, the Fisher
metric is the unique (up to scaling) metric that is invariant under
sufficient statistics (Chentsov's theorem
\cite{amari2016information}).
Since language generation is fundamentally a statistical task, 
producing a distribution over tokens,  the Fisher metric provides
the canonical notion of distance for this purpose.

Perhaps the most striking finding is that the geometric properties
--- hourglass dimension profile, low curvature, linear gap scaling
--- are consistent across GPT-2, OPT, and Pythia, which differ in
training data (WebText, The Pile, The Pile), tokenizer (BPE, BPE,
BPE with different vocabularies), and hyperparameters.
This suggests that the manifold structure is a property of the
next-token prediction task on natural language, not of any
particular architecture or training procedure.
While Valeriani et al.\ \cite{valeriani2023geometry} observed similar
cross-task universality for the intrinsic dimension profile, our
results extend this universality to curvature, the expressibility
gap scaling, and the Voronoi margin distribution,  properties
that were not previously measured.

Our theoretical framework provides explanatory power for several
previously observed but unexplained empirical phenomena.
The effectiveness of low-rank adaptation (LoRA) with small ranks
$r \le 64$ \cite{hu2022lora} is geometrically explained by our
finding that intrinsic dimension peaks at $k \approx 20$--$22$:
weight updates need only span the tangent space of a
$\sim$20-dimensional manifold, not the full $d$-dimensional ambient
space (Section~\ref{sec:implications}).
The correlation between intrinsic dimension and cross-entropy loss
documented by Ferrara et al.\ \cite{ferrara2025geometry} is a
consequence of the rate-distortion bound
(Theorem~\ref{thm:rate-distortion}): higher-dimensional manifolds
require exponentially more Voronoi cells (tokens) to achieve the
same distortion, so positions with locally higher dimension
necessarily incur greater prediction loss.
The finding by Robinson et al.\ \cite{robinson2025token} that token
embeddings violate the manifold hypothesis is consistent with our
observation that layer~0 produces anomalous intrinsic dimension
estimates, while layers~1+ are well-described by manifold structure.
This suggests that the manifold is \emph{created} by the
transformer's processing, not inherited from the embedding layer.

\subsection{When the Manifold Hypothesis Breaks Down}

The manifold hypothesis is precisely that,  a hypothesis,  and
understanding its failure modes is essential for assessing the scope
and limitations of the framework.
We identify four categories of breakdown.

Real representations are unlikely to lie \emph{exactly} on a smooth
manifold.
A more realistic assumption is that the representations concentrate
near a manifold, with deviations due to training noise, finite
precision, and architectural artifacts.
Formally, there may exist a manifold $\mathcal{M}$ such that
$\sup_{h \in \mathcal{H}} d(h, \mathcal{M}) \le \delta$ for some
$\delta > 0$ that may vary across layers.
When $\delta$ is small relative to the curvature radius and the
inter-token distances, the framework remains a good approximation.
When $\delta$ is large, the manifold description loses its
explanatory power, and the representation space is better
characterized as a ``thick'' set rather than a submanifold.
Empirically, this can be diagnosed by comparing the intrinsic
dimension estimates at multiple scales: if the estimated dimension
changes significantly with the neighborhood radius, the data does
not lie on a single well-defined manifold.

The smoothness assumption ($C^\infty$ embedding) may fail at points
where the model's behavior changes abruptly.
Such discontinuities can arise from several sources:
\begin{enumerate}[label=(\roman*)]
\item \emph{Attention head switching}: when a small change in
input causes a different attention head to dominate, the
representation can jump discontinuously.
The argmax operation over attention heads creates non-smooth
boundaries in the representation space.

\item \emph{Polysemy resolution}: when context disambiguates a
polysemous word, the representation may undergo a rapid
transition between distant regions of the space, corresponding
to a fold or cusp in the manifold.

\item \emph{Grokking and memorization}: representations of
memorized training examples may lie off the manifold, forming
isolated clusters or ``spikes'' that violate the local
smoothness assumption.
\end{enumerate}
At such points, the tangent space is not well-defined, geodesics
may not exist or be unique, and the curvature diverges.
The framework should be understood as describing the
\emph{generic} structure of the representation space, with the
non-smooth locus forming a measure-zero set that nonetheless
plays an important role in the model's behavior.

The global manifold structure may be more complex than a single
connected manifold.
The representation set $\mathcal{H}$ could consist of multiple
disconnected components (for example, separate clusters for
different languages or modalities), or it could have nontrivial
topology (loops, handles, or higher-dimensional holes) that
affects the global geodesic structure.
In such cases, a single coordinate chart cannot cover
$\mathcal{M}$, and the geodesic distance between points in
different components may be infinite or undefined.
Topological data analysis (persistent homology) applied to the
hidden states can detect such structure and guide the choice of
an appropriate manifold model.

The Fisher metric $g^F$ may degenerate at points where the token
distribution becomes very peaked (low entropy) or very flat
(high entropy).
When $p(t \mid h) \approx 1$ for some token $t$ (high confidence),
the covariance matrix $\Sigma_p \approx 0$ and the Fisher metric
collapses: the manifold becomes metrically ``flat'' in all
directions, and the Fisher distance between nearby points becomes
negligible even if they differ semantically.
Conversely, when $p(t \mid h) \approx 1/N$ for all $t$ (maximum
entropy), the Fisher metric is maximally non-degenerate but the
semantic content is minimal,  the model is maximally uncertain.
These extremes represent regimes where the Fisher metric provides
a poor proxy for semantic distance.
A practical remedy is to use a regularized metric
$g^{\mathrm{reg}} = g^F + \epsilon \, \bar{g}$ that blends the
Fisher metric with the induced Euclidean metric $\bar{g}$,
ensuring non-degeneracy everywhere at the cost of introducing a
hyperparameter $\epsilon > 0$.

Theorems~\ref{thm:gap-scaling} and~\ref{thm:rate-distortion}
assume a smooth, compact manifold.
When the manifold hypothesis breaks down partially, these results
should be interpreted as approximate bounds that hold in regions
where the manifold structure is locally valid.
The robustness of the bounds can be assessed by computing the
fraction of hidden states that lie within $\delta$ of a
locally-fitted manifold and restricting the volume integrals to
those regions.

Computing the full Fisher information matrix requires summing over
the entire vocabulary $V$, which for modern models with
$N \sim 10^5$ tokens is expensive.
In practice, approximations such as sampling from $p(\cdot \mid h)$
or using the top-$K$ truncation of the token distribution may be
necessary.
The effect of such approximations on the metric and curvature
estimates warrants further study.

\section{Conclusion}
\label{sec:conclusion}

We have developed a rigorous geometric framework for interpreting
the internal representations of large language models as points on
a latent semantic manifold.
The framework equips this manifold with a natural Riemannian
structure derived from the Fisher information of the token
distribution, formalizes token generation as a Voronoi projection,
and interprets layer-wise inference as a dynamical system on the
manifold.

The two main theorems quantify the fundamental limitations of
discrete language as a representation of continuous meaning.
Theorem~\ref{thm:rate-distortion} establishes that any vocabulary
of size $N$ incurs semantic distortion at least $\Omega(N^{-2/k})$,
where $k$ is the intrinsic dimension of the manifold,  an
inescapable consequence of quantizing a $k$-dimensional space.
Theorem~\ref{thm:gap-scaling} shows that the expressibility gap
grows linearly with the margin threshold, with a coefficient
determined by the total area of the Voronoi boundaries and the
sharpness of the model's token decisions.

Comprehensive experiments across six models from three architectural
families (GPT-2, OPT, Pythia) at two scales (124M--160M and
1B--1.5B parameters) validate the theoretical predictions:
intrinsic dimensions of $k \approx 15$--$22$ (utilization 1--3\%),
smooth curvature profiles, and linear expressibility gap scaling
with slopes $0.87$--$1.12$ ($R^2 > 0.985$).
The architecture-agnostic nature of these findings suggests that
the geometric structure reflects fundamental properties of the
next-token prediction task rather than artifacts of particular
model designs.

The central message is that language is a lossy compression of a
continuous conceptual space, and that the geometry of this
compression,  its metric structure, curvature, boundaries, and
dynamics,  provides a rich mathematical language for understanding
both the capabilities and the limitations of language models.
The implications for LLM development (Section~\ref{sec:implications})
translate these geometric insights into concrete recommendations
for architecture design, model compression, training diagnostics,
decoding strategies, and scaling laws.
At the same time, we have identified concrete failure modes of the
manifold hypothesis,  including non-smoothness at attention
switching boundaries, topological complexity, and Fisher metric
degeneracy,  that delineate the scope of applicability of the
framework.

Future work should extend the empirical validation to larger models
(10B+ parameters) and different modalities (vision-language models),
develop efficient algorithms for computing geodesics on the
estimated manifold, and pursue the geometry-aware architectural
modifications proposed in Section~\ref{sec:implications}.
The connection between intrinsic dimension and scaling laws
deserves particular attention: if the manifold dimension $k$ scales
with model size, it would fundamentally change the predictions of
rate-distortion theory for vocabulary design.

\bibliographystyle{plain}
\bibliography{references}

@article{amari1998natural,
  author    = {S. Amari},
  title     = {Natural gradient works efficiently in learning},
  journal   = {Neural Computation},
  volume    = {10},
  number    = {2},
  pages     = {251--276},
  year      = {1998}
}

@book{amari2016information,
  author    = {S. Amari},
  title     = {Information Geometry and Its Applications},
  publisher = {Springer},
  year      = {2016}
}

@inproceedings{aghajanyan2021intrinsic,
  author    = {A. Aghajanyan and S. Gupta and L. Zettlemoyer},
  title     = {Intrinsic dimensionality explains the effectiveness of language model fine-tuning},
  booktitle = {Proceedings of the Association for Computational Linguistics (ACL)},
  year      = {2021}
}

@inproceedings{ansuini2019intrinsic,
  author    = {A. Ansuini and A. Laio and J. H. Macke and D. Zoccolan},
  title     = {Intrinsic dimension of data representations in deep neural networks},
  booktitle = {Advances in Neural Information Processing Systems (NeurIPS)},
  year      = {2019}
}

@inproceedings{bai2019deep,
  author    = {S. Bai and J. Z. Kolter and V. Koltun},
  title     = {Deep equilibrium models},
  booktitle = {Advances in Neural Information Processing Systems (NeurIPS)},
  year      = {2019}
}

@article{belkin2003laplacian,
  author    = {M. Belkin and P. Niyogi},
  title     = {Laplacian eigenmaps for dimensionality reduction and data representation},
  journal   = {Neural Computation},
  volume    = {15},
  number    = {6},
  pages     = {1373--1396},
  year      = {2003}
}

@article{brahma2016deep,
  author    = {P. P. Brahma and D. Wu and Y. She},
  title     = {Why deep learning works: A manifold disentanglement perspective},
  journal   = {IEEE Transactions on Neural Networks and Learning Systems},
  volume    = {27},
  number    = {10},
  pages     = {1997--2008},
  year      = {2016}
}

@inproceedings{brown2020language,
  author    = {T. Brown and others},
  title     = {Language models are few-shot learners},
  booktitle = {Advances in Neural Information Processing Systems (NeurIPS)},
  year      = {2020}
}

@inproceedings{chen2018neural,
  author    = {R. T. Q. Chen and Y. Rubanova and J. Bettencourt and D. Duvenaud},
  title     = {Neural ordinary differential equations},
  booktitle = {Advances in Neural Information Processing Systems (NeurIPS)},
  year      = {2018}
}

@book{cover2006elements,
  author    = {T. M. Cover and J. A. Thomas},
  title     = {Elements of Information Theory},
  publisher = {Wiley},
  edition   = {2nd},
  year      = {2006}
}

@article{elhage2021mathematical,
  author    = {N. Elhage and others},
  title     = {A mathematical framework for transformer circuits},
  journal   = {Transformer Circuits Thread},
  year      = {2021}
}

@article{facco2017estimating,
  author    = {E. Facco and M. d'Errico and A. Rodriguez and A. Laio},
  title     = {Estimating the intrinsic dimension of datasets by a minimal neighborhood information},
  journal   = {Scientific Reports},
  volume    = {7},
  pages     = {12140},
  year      = {2017}
}

@article{fefferman2016testing,
  author    = {C. Fefferman and S. Mitter and H. Narayanan},
  title     = {Testing the manifold hypothesis},
  journal   = {Journal of the American Mathematical Society},
  volume    = {29},
  number    = {4},
  pages     = {983--1049},
  year      = {2016}
}

@book{graf2000foundations,
  author    = {S. Graf and H. Luschgy},
  title     = {Foundations of Quantization for Probability Distributions},
  publisher = {Springer},
  volume    = {1730},
  year      = {2000}
}

@article{geshkovski2024mathematical,
  author    = {B. Geshkovski and C. Letrouit and Y. Polyanskiy and P. Rigollet},
  title     = {A mathematical perspective on transformers},
  journal   = {Bulletin of the American Mathematical Society},
  volume    = {62},
  number    = {3},
  pages     = {427--479},
  year      = {2025}
}

@inproceedings{hoffmann2022training,
  author    = {J. Hoffmann and others},
  title     = {Training compute-optimal large language models},
  booktitle = {Advances in Neural Information Processing Systems (NeurIPS)},
  year      = {2022}
}

@inproceedings{hu2022lora,
  author    = {E. J. Hu and others},
  title     = {{LoRA}: Low-rank adaptation of large language models},
  booktitle = {International Conference on Learning Representations (ICLR)},
  year      = {2022}
}

@article{kaplan2020scaling,
  author    = {J. Kaplan and others},
  title     = {Scaling laws for neural language models},
  journal   = {arXiv preprint arXiv:2001.08361},
  year      = {2020}
}

@inproceedings{kudo2018sentencepiece,
  author    = {T. Kudo and J. Richardson},
  title     = {{SentencePiece}: A simple and language independent subword tokenizer and detokenizer for neural text processing},
  booktitle = {Proceedings of the Conference on Empirical Methods in Natural Language Processing (EMNLP)},
  year      = {2018}
}

@inproceedings{levina2004maximum,
  author    = {E. Levina and P. J. Bickel},
  title     = {Maximum likelihood estimation of intrinsic dimension},
  booktitle = {Advances in Neural Information Processing Systems (NeurIPS)},
  year      = {2004}
}

@inproceedings{liang2019fisher,
  author    = {T. Liang and T. Poggio and A. Rakhlin and J. Stokes},
  title     = {{Fisher--Rao} metric, geometry, and complexity of neural networks},
  booktitle = {International Conference on Artificial Intelligence and Statistics (AISTATS)},
  year      = {2019}
}

@article{mcinnes2018umap,
  author    = {L. McInnes and J. Healy and J. Melville},
  title     = {{UMAP}: Uniform manifold approximation and projection for dimension reduction},
  journal   = {arXiv preprint arXiv:1802.03426},
  year      = {2018}
}

@inproceedings{merity2017pointer,
  author    = {S. Merity and C. Xiong and J. Bradbury and R. Socher},
  title     = {Pointer sentinel mixture models},
  booktitle = {International Conference on Learning Representations (ICLR)},
  year      = {2017}
}

@article{openai2023gpt4,
  author    = {{OpenAI}},
  title     = {{GPT-4} technical report},
  journal   = {arXiv preprint arXiv:2303.08774},
  year      = {2023}
}

@article{rao1945information,
  author    = {C. R. Rao},
  title     = {Information and the accuracy attainable in the estimation of statistical parameters},
  journal   = {Bulletin of the Calcutta Mathematical Society},
  volume    = {37},
  pages     = {81--91},
  year      = {1945}
}

@inproceedings{razavi2019generating,
  author    = {A. Razavi and A. {van den Oord} and O. Vinyals},
  title     = {Generating diverse high-fidelity images with {VQ-VAE-2}},
  booktitle = {Advances in Neural Information Processing Systems (NeurIPS)},
  year      = {2019}
}

@article{roweis2000nonlinear,
  author    = {S. T. Roweis and L. Saul},
  title     = {Nonlinear dimensionality reduction by locally linear embedding},
  journal   = {Science},
  volume    = {290},
  number    = {5500},
  pages     = {2323--2326},
  year      = {2000}
}

@article{tenenbaum2000global,
  author    = {J. B. Tenenbaum and V. {de Silva} and G. Langford},
  title     = {A global geometric framework for nonlinear dimensionality reduction},
  journal   = {Science},
  volume    = {290},
  number    = {5500},
  pages     = {2319--2323},
  year      = {2000}
}

@article{touvron2023llama,
  author    = {H. Touvron and others},
  title     = {{LLaMA}: Open and efficient foundation language models},
  journal   = {arXiv preprint arXiv:2302.13971},
  year      = {2023}
}

@inproceedings{valeriani2023geometry,
  author    = {D. Valeriani and D. Doimo and F. Cuturello and A. Laio and A. Ansuini and A. Cazzaniga},
  title     = {The geometry of hidden representations of large transformer models},
  booktitle = {Advances in Neural Information Processing Systems (NeurIPS)},
  year      = {2023}
}

@inproceedings{ferrara2025geometry,
  author    = {A. Ferrara and N. Ferrara and B. Luo},
  title     = {The geometry of tokens in internal representations of large language models},
  booktitle = {International Conference on Learning Representations (ICLR)},
  year      = {2025}
}

@article{robinson2025token,
  author    = {J. Robinson and M. Matic and D. Doimo and A. Laio},
  title     = {Token embeddings violate the manifold hypothesis},
  journal   = {arXiv preprint arXiv:2504.01002},
  year      = {2025}
}

@article{huang2025radio,
  author    = {J. Huang and Y. Zhu and X. Ma},
  title     = {{RADIO}: Rate--distortion optimization for large language model compression},
  journal   = {arXiv preprint arXiv:2505.03031},
  year      = {2025}
}

@inproceedings{van2017neural,
  author    = {A. {van den Oord} and O. Vinyals and K. Kavukcuoglu},
  title     = {Neural discrete representation learning},
  booktitle = {Advances in Neural Information Processing Systems (NeurIPS)},
  year      = {2017}
}

\end{document}